\documentclass{article}
\usepackage[nonatbib,final]{neurips_2020}
\usepackage[utf8]{inputenc} 
\usepackage[T1]{fontenc}    
\usepackage{hyperref}       
\usepackage{url}            
\usepackage{booktabs}       
\usepackage{amsfonts}       
\usepackage{nicefrac}       
\usepackage{microtype}      
\usepackage{bm}
\usepackage{algorithm}
\usepackage{multirow}
\usepackage{times}
\usepackage{epsfig}
\usepackage{amsmath} 
\usepackage{amssymb}
\usepackage{subfigure}
\usepackage{diagbox}
\usepackage{makecell}
\usepackage{colortbl}
\usepackage{threeparttable}
\usepackage{wrapfig}
\usepackage[dvipsnames]{xcolor}

\usepackage[square,numbers,sort&compress]{natbib}
\hypersetup{
    colorlinks,
    citecolor=RoyalBlue,
}

\usepackage{tikz}
\usepackage{calc}

\usepackage{amsmath,amssymb,amsthm,array,amsfonts,mathtools}

\usepackage{algpseudocode}  
\usepackage{amsmath}
\usepackage{verbatim}

\newtheorem{theorem}{Theorem}

\newtheorem{corollary}[theorem]{Corollary}

\title{Deep Multimodal Fusion by Channel Exchanging}

\author{Yikai Wang$^1$, Wenbing Huang$^1$, Fuchun Sun$^1$\thanks{Corresponding author: Fuchun Sun.}$\;$,
Tingyang Xu$^2$, Yu Rong$^2$, Junzhou Huang$^2$
\\$^1$Beijing National Research Center for Information Science and Technology$\,$(BNRist),\\ State Key Lab on Intelligent Technology and Systems,\\ Department of Computer Science and Technology, Tsinghua University $\;^2$Tencent AI Lab\\ \texttt{\footnotesize wangyk17@mails.tsinghua.edu.cn, hwenbing@126.com, fcsun@tsinghua.edu.cn,}\\\texttt{\footnotesize tingyangxu@tencent.com, yu.rong@hotmail.com, jzhuang@uta.edu}}

\begin{document}
\maketitle

\begin{abstract}

Deep multimodal fusion by using multiple sources of data for classification or regression has exhibited a clear advantage over the unimodal counterpart on various applications. Yet, current methods including aggregation-based and alignment-based fusion are still inadequate in balancing the trade-off between inter-modal fusion and intra-modal processing, incurring a bottleneck of performance improvement. To this end, this paper proposes Channel-Exchanging-Network (CEN), a parameter-free multimodal fusion framework that dynamically exchanges channels between sub-networks of different modalities. Specifically, the channel exchanging process is self-guided by individual channel importance that is measured by the magnitude of Batch-Normalization (BN) scaling factor during training. The validity of such exchanging process is also guaranteed by sharing convolutional filters yet keeping separate BN layers across modalities, which, as an add-on benefit, allows our multimodal architecture to be almost as compact as a unimodal network. Extensive experiments on semantic segmentation via RGB-D data and image translation through multi-domain input verify the effectiveness of our CEN compared to current state-of-the-art methods. Detailed ablation studies have also been carried out, which provably affirm the advantage of each component we propose. Our code is available at \url{https://github.com/yikaiw/CEN}.

\end{abstract}

\section{Introduction}

Encouraged by the growing availability of low-cost sensors, \emph{multimodal fusion} that takes advantage of data obtained from different sources/structures for classification or regression has become a central problem in machine learning~\cite{journals/pami/BaltrusaitisAM19}. Joining the success of deep learning, multimodal fusion is recently specified as \emph{deep multimodal fusion} by introducing end-to-end neural integration of multiple modalities~\cite{ramachandram2017deep}, and it has exhibited remarkable benefits 
against the unimodal paradigm in semantic segmentation~\cite{conf/iccv/LeePH17,journals/ijcv/ValadaMB20}, action recognition~\cite{fan2018end,conf/eccv/GarciaMM18,journals/tip/SongLLG20}, visual question answering~\cite{conf/iccv/AntolALMBZP15,conf/nips/IlievskiF17}, and many others~\cite{conf/iccv/BalntasDSSKK17,conf/iclr/JinYBJ19,conf/iccv/ZhangZSWSL19}.

A variety of works have been done towards deep multimodal fusion~\cite{ramachandram2017deep}. Regarding the type of how they fuse, existing methods are generally categorized into \emph{aggregation-based} fusion, \emph{alignment-based} fusion, and the mixture of them~\cite{journals/pami/BaltrusaitisAM19}. The aggregation-based methods employ a certain operation (\emph{e.g.} averaging \cite{conf/accv/HazirbasMDC16}, concatenation~\cite{conf/icml/NgiamKKNLN11,conf/cvpr/ZengTHYSCW19}, and self-attention~\cite{journals/ijcv/ValadaMB20}) to combine multimodal sub-networks into a single network. The alignment-based fusion~\cite{conf/cvpr/ChengCLZH17,journals/tip/SongLLG20,conf/eccv/WangWTSW16}, instead, adopts a regulation loss to align the embedding of all sub-networks while keeping full propagation for each of them. The difference between such two mechanisms is depicted in Figure~\ref{pic:sketch}. Another categorization of multimodal fusion can be specified as early, middle, and late fusion, depending on when to fuse, which have been discussed in earlier works~\cite{atrey2010multimodal,bruni2014multimodal,hall1997introduction,snoek2005early} and also in the deep learning literature~\cite{journals/pami/BaltrusaitisAM19,kiela2017deep,lazaridou2014wampimuk,conf/nips/VriesSMLPC17}.

Albeit the fruitful progress, it remains a great challenge on how to integrate the common information across modalities, meanwhile preserving the specific patterns of each one. In particular, the aggregation-based fusion is prone to underestimating the intra-modal propagation once the multimodal sub-networks have been aggregated. On the contrary, the alignment-based fusion maintains the intra-modal propagation, but it always delivers ineffective inter-modal fusion owing to the weak message exchanging by solely training the alignment loss. To balance between inter-modal fusion and intra-modal processing, current methods usually resort to careful hierarchical combination of the aggregation and alignment fusion for enhanced performance, at a cost of extra computation and engineering overhead~\cite{conf/cvpr/DuWWZW19,conf/iccv/LeePH17,conf/cvpr/ZengTHYSCW19}. 

\textbf{Present Work.}
We propose Channel-Exchanging-Network (CEN) which is parameter-free, adaptive, and effective. Instead of using aggregation or alignment as before, CEN dynamically exchanges the channels between sub-networks for fusion (see Figure~\ref{pic:sketch}(c)). The core of CEN lies in its smaller-norm-less-informative assumption inspired from network pruning~\cite{conf/iccv/LiuLSHYZ17,conf/iclr/YeL0W18}. To be specific, we utilize the scaling factor (\emph{i.e.} $\gamma$) of Batch-Normalization (BN)~\cite{conf/icml/IoffeS15} as the importance measurement of each corresponding channel, and replace the channels associated with close-to-zero factors of each modality with the mean of other modalities. Such message exchanging is parameter-free and self-adaptive, as it is dynamically controlled by the scaling factors that are determined by the training itself. 
Besides, we only allow directed channel exchanging within a certain range of channels in each modality to preserve intra-modal processing. More details are provided in \textsection~\ref{subsec:channel_exc}.
Necessary theories on the validity of our idea are also presented in \textsection~\ref{subsec_anslysis}.

Another hallmark of CEN is that the parameters except BN layers of all sub-networks are shared with each other (\textsection~\ref{subsec:net_sharing}). Although this idea is previously studied in~\cite{conf/cvpr/ChangYSKH19,wang2020asymfusion}, we apply it here to serve specific purposes in CEN: by using private BNs, as already discussed above, we can determine the channel importance for each individual modality;
by sharing convolutional filters, the corresponding channels among different modalities are embedded with the same mapping, thus more capable of modeling the modality-common statistic. This design further compacts the multimodal architecture to be almost as small as the unimodal one.

We evaluate our CEN on two studies: semantic segmentation via RGB-D data~\cite{conf/eccv/SilbermanHKF12,conf/cvpr/SongLX15} and image translation through multi-domain input~\cite{conf/cvpr/ZamirSSGMS18}. 
It demonstrates that CEN yields remarkably superior performance than various kinds of fusion methods based on aggregation or alignment under a fair condition of comparison. In terms of semantic segmentation particularly, our CEN significantly outperforms state-of-the-art methods on two popular benchmarks. We also conduct ablation studies to isolate the benefit of each proposed component. More specifications are provided in \textsection~\ref{sec:experiments}.

\begin{figure}[t]
\centering
\includegraphics[scale=0.27]{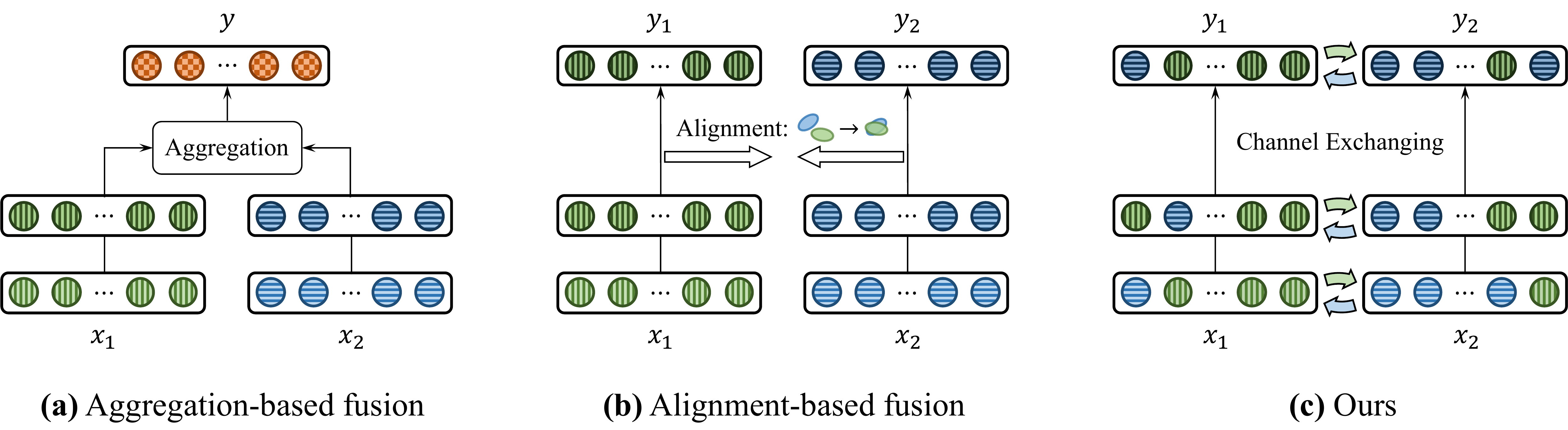}
\caption{A sketched comparison between existing fusion methods and ours.}
\label{pic:sketch}
\end{figure}

\section{Related Work}
We introduce the methods of deep multimodal fusion, and the concepts related to our paper.

\textbf{Deep multimodal fusion.}
As discussed in introduction, deep multimodal fusion methods can be mainly categorized into aggregation-based fusion and alignment-based fusion~\cite{journals/pami/BaltrusaitisAM19}.
Due to the weakness in intra-modal processing, recent aggregation-based works perform feature fusion while still maintaining the sub-networks of all modalities~\cite{conf/cvpr/DuWWZW19,conf/iccv/LinCCHH17}. 
Besides, \cite{conf/accv/HazirbasMDC16} points out the performance by fusion is highly affected by the choice of which layer to fuse. Alignment-based fusion methods align multimodal features by applying the similarity regulation, where Maximum-Mean-Discrepancy (MMD)~\cite{journals/jmlr/GrettonBRSS12} is usually adopted for the measurement.
However, simply focusing on unifying the whole distribution may overlook the specific patterns in each domain/modality \cite{conf/nips/BousmalisTSKE16,journals/tip/SongLLG20}. Hence, \cite{conf/eccv/WangWTSW16} provides a way that may alleviate this issue, which correlates modality-common features while simultaneously maintaining modality-specific information. 
There is also a portion of the multimodal learning literature based on modulation~\cite{de2017guesswhat,dumoulin2018feature,conf/nips/VriesSMLPC17}. Different from these types of fusion methods, we propose a new fusion method by channel exchanging, which potentially enjoys the guarantee to both sufficient inter-model interactions and intra-modal learning. 

\textbf{Other related concepts.} 
The idea of using BN scaling factor to evaluate the importance of CNN channels has been studied in network pruning~\cite{conf/iccv/LiuLSHYZ17,conf/iclr/YeL0W18} and representation learning~\cite{shao2020channel}. 
Moreover, \cite{conf/iccv/LiuLSHYZ17} enforces $\ell_1$ norm penalty on the scaling factors and explicitly prunes out filters meeting a sparsity criteria. Here, we apply this idea as an adaptive tool to determine where to exchange and fuse.  CBN~\cite{conf/nips/VriesSMLPC17} performs cross-modal message passing by modulating BN of one modality conditional on the other, which is clearly different from our method that directly exchanges channels between different modalities for fusion. 
ShuffleNet~\cite{conf/cvpr/ZhangZLS18} proposes to shuffle a portion of channels among multiple groups for efficient propagation in light-weight networks, which is similar to our idea of exchanging channels for message fusion. Yet, while the motivation of our paper is highly different, the exchanging process is self-determined by the BN scaling factors, instead of the random exchanging in ShuffleNet. 

\section{Channel Exchanging Networks}
In this section, we introduce our CEN, by mainly specifying its two fundamental components: the channel exchanging process and the sub-network sharing mechanism, followed by necessary analyses.

\subsection{Problem Definition}
Suppose we have the $i$-th input data of $M$ modalities, $\bm{x}^{(i)}=\{\bm{x}_m^{(i)}\in\mathbb{R}^{C\times (H\times W)}\}_{m=1}^M$, where $C$ denotes the number of channels, $H$ and $W$ denote the height and width of the feature map\footnote{Although our paper is specifically interested in image data, our method is still general to other domains; for example, we can set $H=W=1$ for vectors.}. We define $N$ as the batch-size.
The goal of deep multimodal fusion is to determine a multi-layer network $f(\bm{x}^{(i)})$ (particularly CNN in this paper) whose output $\hat{\bm{y}}^{(i)}$ is expected to fit the target $\bm{y}^{(i)}$ as much as possible. This can be implemented by minimizing the empirical loss as
\begin{equation}
\label{eq:loss}
    \min_{f}\frac{1}{N}\sum_{i=1}^{N}\mathcal{L}\left(\hat{\bm{y}}^{(i)}=f(\bm{x}^{(i)}),\bm{y}^{(i)}\right).
\end{equation}

We now introduce two typical kinds of instantiations to Equation~\ref{eq:loss}:

\textbf{I.}
The aggregation-based fusion first processes each $m$-th modality with a separate sub-network $f_m$ and then combine all their outputs via an aggregation operation followed by a global mapping. In formal, it computes the output by
\begin{equation}
    \hat{\bm{y}}^{(i)}=f(\bm{x}^{(i)})=h(\text{Agg}(f_1(\bm{x}_1^{(i)}),\cdots,f_M(\bm{x}_M^{(i)}))),
\end{equation}
where $h$ is the global network and $\text{Agg}$ is the aggregation function. The aggregation can be implemented as averaging~\cite{conf/accv/HazirbasMDC16}, concatenation~\cite{conf/cvpr/ZengTHYSCW19}, and self-attention~\cite{journals/ijcv/ValadaMB20}. All networks are optimized via minimizing Equation~\ref{eq:loss}.  

\textbf{II.}
The alignment-based fusion leverages an alignment loss for capturing the inter-modal concordance while keeping the outputs of all sub-networks $f_m$. Formally, it solves
\begin{equation}
\label{eq:alig}
    \min_{f_{1:M}} \frac{1}{N}\sum_{i=1}^{N} \mathcal{L}\left(\sum_{m=1}^M\alpha_mf_m(\bm{x}_m^{(i)}), \bm{y}^{(i)}\right)+ \text{Alig}_{f_{1:M}}(\bm{x}^{(i)}), \quad s.t. \sum_{m=1}^M\alpha_m=1,
\end{equation}
where the alignment $\text{Alig}_{f_{1:M}}$ is usually specified as Maximum-Mean-Discrepancy (MMD)~\cite{journals/jmlr/GrettonBRSS12} between certain hidden features of sub-networks, and the final output $\sum_{m=1}^M\alpha_mf_m(\bm{x}_m^{(i)})$ is an ensemble of $f_m$ associated with the decision score $\alpha_m$ which is learnt by an additional softmax output to meet the simplex constraint. 

As already discussed in introduction, both fusion methods are insufficient to determine the trade-off between fusing modality-common information and preserving modality-specific patterns. In contrast, our CEN is able to combine their best, the details of which are clarified in the next sub-section. 

\subsection{Overall Framework}
The whole optimization objective of our method is
\begin{equation}
\label{eq:cen}
    \min_{f_{1:M}} \frac{1}{N}\sum_{i=1}^{N} \mathcal{L}\left(\sum_{m=1}^M\alpha_mf_m(\bm{x}^{(i)}), \bm{y}^{(i)}\right)+ \lambda\sum_{m=1}^{M}\sum_{l=1}^L|\bm{\hat{\gamma}}_{m,l}|, \quad s.t. \sum_{m=1}^M\alpha_m=1,
\end{equation}
where,
\vspace{-0.3em}
\begin{itemize}
    \item The sub-network $f_m(\bm{x}^{(i)})$ (opposed to $f_m(\bm{x}_m^{(i)})$ in Equation~\ref{eq:alig} of the alignment fusion) fuses multimodal information by channel exchanging, as we will detail in \textsection~\ref{subsec:channel_exc};
    \item Each sub-network is equipped with BN layers containing the scaling factors $\bm{\gamma}_{m,l}$ for the $l$-th layer, and we will penalize the  
    $\ell_1$ norm of their certain portion $\bm{\hat{\gamma}}_{m,l}$ for sparsity, which is presented in \textsection~\ref{subsec:channel_exc};
    \item The sub-network $f_m$ shares the same parameters except BN layers to facilitate the channel exchanging as well as to compact the architecture further, as introduced in \textsection~\ref{subsec:net_sharing};
    \item The decision scores of the ensemble output, $\alpha_m$, are trained by a softmax output similar to the alignment-based methods.  
\end{itemize}

By the design of Equation~\ref{eq:cen}, we conduct a parameter-free message fusion across modalities while maintaining the self-propagation of each sub-network so as to characterize the specific statistic of each modality. Moreover, our fusion of channel exchanging is self-adaptive and easily embedded to everywhere of the sub-networks, with the details given in what follows.

\begin{figure}[t]
\centering
\includegraphics[scale=0.3]{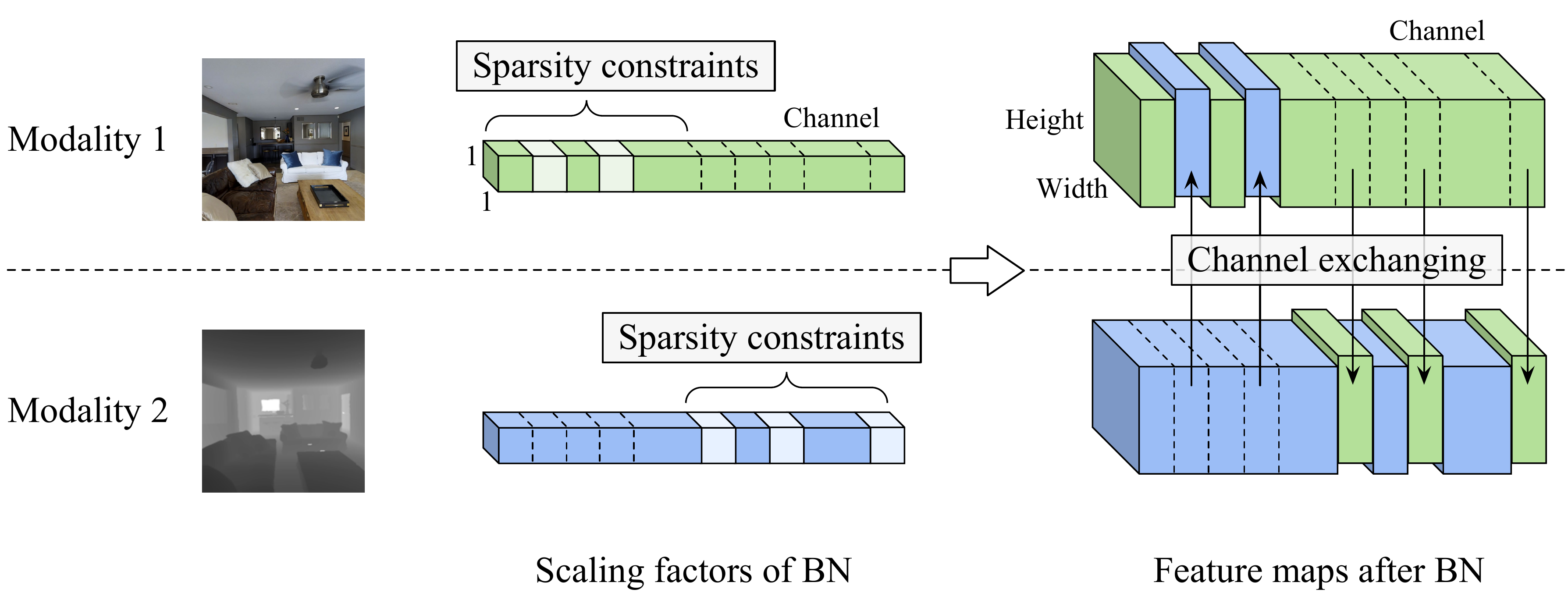}
\vskip -0.05 in
\caption{An illustration of our multimodal fusion strategy. The sparsity constraints on scaling factors are applied to disjoint regions of different modalities. A feature map will be replaced by that of other modalities at the same position, if its scaling factor is lower than a threshold. }
\label{channel}
\vskip -0.1 in
\end{figure}

\subsection{Channel Exchanging by Comparing BN Scaling Factor}
\label{subsec:channel_exc}
Prior to introducing the channel exchanging process, we first review the BN layer~\cite{conf/icml/IoffeS15}, which is used widely in deep learning to eliminate covariate shift and improve generalization. We denote by $\bm{x}_{m,l}$ the $l$-th layer feature maps of the $m$-th sub-network, and by $\bm{x}_{m,l,c}$ the $c$-th channel. The BN layer performs a  normalization of $\bm{x}_{m,l}$ followed by an affine transformation, namely,
\begin{equation}
\label{eq:bn}
\bm{x}'_{m,l,c} = \gamma_{m,l,c}\frac{\bm{x}_{m,l,c}-\mu_{m,l,c}}{\sqrt{\sigma^2_{m,l,c}+\epsilon}}+\beta_{m,l,c},
\end{equation}
where, $\mu_{m,l,c}$ and $\sigma_{m,l,c}$ compute the mean and the standard deviation, respectively, of all activations over all pixel locations ($H$ and $W$) for the current mini-batch data; $\gamma_{m,l,c}$ and $\beta_{m,l,c}$ are the trainable scaling factor and offset, respectively; $\epsilon$ is a small constant to avoid divisions by zero. The $(l+1)$-th layer takes $\{\bm{x}'_{m,l,c}\}_c$ as input after a non-linear function.

The factor $\gamma_{m,l,c}$ in Equation~\ref{eq:bn} evaluates the correlation between the input $\bm{x}_{m,l,c}$ and the output $\bm{x}'_{m,l,c}$ during training. The gradient of the loss \emph{w.r.t.} $\bm{x}_{m,l,c}$ will approach 0 if $\gamma_{m,l,c}\rightarrow 0$, implying that $\bm{x}_{m,l,c}$ will lose its influence to the final prediction and become redundant thereby. Moreover, we will prove in \textsection~\ref{subsec_anslysis} that the state of $\gamma_{m,l,c}=0$ is attractive with a high probability, given the $\ell_1$ norm regulation in Equation~\ref{eq:cen}. In other words, once the current channel $\bm{x}_{m,l,c}$ becomes redundant due to $\gamma_{m,l,c}\rightarrow 0$ at a certain training step, it will almost do henceforth. 

It thus motivates us to replace the channels of small scaling factors with the ones of other sub-networks, since those channels potentially are redundant. To do so, we derive
\begin{equation}
\label{eq:exchange-bn}
\bm{x}'_{m,l,c} =
\begin{cases}
\gamma_{m,l,c}\frac{\bm{x}_{m,l,c}-\mu_{m,l,c}}{\sqrt{\sigma^2_{m,l,c}+\epsilon}}+\beta_{m,l,c}, &\text{if}\quad \gamma_{m,l,c}>\theta; \\
\frac{1}{M-1}\sum\limits_{m'\neq m}^M\gamma_{m',l,c}\frac{\bm{x}_{m',l,c}-\mu_{m',l,c}}{\sqrt{\sigma^2_{m',l,c}+\epsilon}}+\beta_{m',l,c}, &\text{else};
\end{cases}
\end{equation}
where, the current channel is replaced with the mean of other channels if its scaling factor is smaller than a certain threshold $\theta\approx0^+$. In a nutshell, if one channel of one modality has little impact to the final prediction, then we replace it with the mean of other modalities. We apply Equation~\ref{eq:exchange-bn} for each modality before feeding them into the nonlinear activation followed by the convolutions in the next layer. Gradients are detached from the replaced channel and back-propagated through the new ones.
    
In our implementation, we divide the whole channels into $M$ equal sub-parts, and only perform the channel exchanging in each different sub-part for different modality. We denote the scaling factors that are allowed to be replaced as $\bm{\hat{\gamma}}_{m,l}$. We further impose the sparsity constraint on $\bm{\hat{\gamma}}_{m,l}$ in Equation~\ref{eq:cen} to discover unnecessary channels. As the exchanging in Equation~\ref{eq:exchange-bn} is a directed process within only one sub-part of channels, it hopefully can not only retain modal-specific propagation in the other $M-1$ sub-parts but also avoid unavailing exchanging since $\gamma_{m',l,c}$, different from $\hat{\gamma}_{m,l,c}$, is out of the sparsity constraint. Figure~\ref{channel} illustrates our channel exchanging process.

\subsection{Sub-Network Sharing with Independent BN}
\label{subsec:net_sharing}
It is known in~\cite{conf/cvpr/ChangYSKH19,wang2020asymfusion} that leveraging private BN layers is able to characterize the traits of different domains or modalities. In our method, specifically, different scaling factors (Equation~\ref{eq:bn}) evaluate the importance of the channels of different modalities, and they should be decoupled.

With the exception of BN layers, all sub-networks $f_m$ share all parameters with each other including convolutional filters\footnote{If the input channels of different modalities are different (\emph{e.g.} RGB and depth), we will broaden their sizes to be the same as their Least Common Multiple (LCM).}. The hope is that we can further reduce the network complexity and therefore improve the predictive generalization. Rather, considering the specific design of our framework, sharing convolutional filters is able to capture the common patterns in different modalities, which is a crucial purpose of multimodal fusion. In our experiments, we conduct multimodal fusion on RGB-D images or on other domains of images corresponding to the same image content. In this scenario, all modalities are homogeneous in the sense that they are just different views of the same input. Thus, sharing parameters between different sub-networks still yields promisingly expressive power. Nevertheless, when we are dealing with heterogeneous modalities (\emph{e.g.} images with text sequences), it would impede the expressive power of the sub-networks if keeping sharing their parameters, hence a more dexterous mechanism is suggested, the discussion of which is left for future exploration.

\subsection{Analysis}
\label{subsec_anslysis}
\begin{theorem}
\label{th:theorem}
Suppose $\{\gamma_{m,l,c}\}_{m,l,c}$ are the BN scaling factors of any multimodal fusion network (without channel exchanging) optimized by Equation~\ref{eq:cen}. Then the probability of $\gamma_{m,l,c}$ being attracted to $\gamma_{m,l,c}=0$  during training (\emph{a.k.a.} $\gamma_{m,l,c}=0$ is the local minimum) is equal to $2\Phi(\lambda|\frac{\partial L}{\partial \bm{x}'_{m,l,c}}|^{-1})-1$, where $\Phi$ derives the cumulative probability of standard Gaussian.
\end{theorem}

In practice, especially when approaching the convergence point, the magnitude of $\frac{\partial L}{\partial \bm{x}'_{m,l,c}}$ is usually very close to zero, indicating that the probability of staying around $\gamma_{m,l,c}=0$ is large. In other words, when the scaling factor of one channel is equal to zero, this channel will almost become redundant during later training process, which will be verified by our experiment in the appendix. Therefore, replacing the channels of $\gamma_{m,l,c}=0$ with other channels (or anything else) will only enhance the trainablity of the model. We immediately have the following corollary,
\setcounter{theorem}{0}
\begin{corollary}
\label{th:corollary}
\vspace{-0.1em}
If the minimal of Equation~\ref{eq:cen} implies $\gamma_{m,l,c}=0$, then the channel exchanging by Equation~\ref{eq:exchange-bn} (assumed no crossmodal parameter sharing) will only decrease the training loss, \emph{i.e.} $\min_{f'_{1:M}}L\leq\min_{f_{1:M}}L$, given the sufficiently expressive $f'_{1:M}$ and $f_{1:M}$ which denote the cases with and without channel exchanging, respectively.
\end{corollary}

\section{Experiments}
\label{sec:experiments}
We contrast the performance of CEN against existing multimodal fusion methods on two different tasks: semantic segmentation and image-to-image translation. The frameworks for both tasks are in the encoder-decoder style. Note that we only perform multimodal fusion within the encoders of different modalities throughout the experiments. Our codes are complied on PyTorch~\cite{conf/nips/PaszkeGMLBCKLGA19}.

\begin{table}[t]
\centering

\caption{Detailed results for different versions of our CEN on NYUDv2. All results are obtained with the backbone RefineNet (ResNet101) of single-scale evaluation for test.}
\vspace{-0.3em}
\resizebox{115mm}{!}{
\begin{tabular}{llll|p{1.2cm}<{\centering}p{1.2cm}<{\centering}p{1.4cm}<{\centering}}
\toprule
\multirow{2}*{Convs}&\multirow{2}*{\makecell[l]{BNs}}&\multirow{2}*{$\ell_1$ Regulation}&\multirow{2}*{Exchange}&\multicolumn{3}{c}{Mean IoU (\%)}\\
&&&&RGB & Depth &Ensemble$\,$\\

\midrule
Unshared&Unshared&$\;\;\;\;\;\;\;\;\times$&$\;\;\;\;\;\,\times$&45.5&35.8&47.6\\
Shared & Shared&$\;\;\;\;\;\;\;\;\times$&$\;\;\;\;\;\,\times$&{43.7}&{35.5}&{45.2}\\
Shared&Unshared&$\;\;\;\;\;\;\;\;\times$&$\;\;\;\;\;\,\times$&46.2&38.4&48.0\\
\midrule
Shared &Unshared  & Half-channel&$\;\;\;\;\;\,\times$&46.0&38.1&47.7\\
Shared &Unshared  & Half-channel &$\;\;\;\;\;\,\checkmark$&\textbf{49.7}&\textbf{45.1}&\textbf{51.1}\\
Shared &Unshared &All-channel&$\;\;\;\;\;\,\checkmark$&48.6&{39.0}&49.8\\
\bottomrule
\end{tabular}}
\label{tabs:component}
\vskip -0.01 in
\end{table}

\begin{table}[t]
\centering

\caption{Comparison with three typical fusion methods including concatenation (concat), fusion by alignment (align), and self-attention (self-att.) on NYUDv2. All results are obtained with the backbone RefineNet (ResNet101) of single-scale evaluation for test.}
\vspace{-0.3em}
\resizebox{136mm}{!}{
\begin{tabular}{ll|cc|cc|c}

\toprule
\multirow{3}*{Modality}&\multirow{3}*{Approach}&\multicolumn{2}{c|}{Commonly-used setting}&\multicolumn{2}{c|}{Same with our setting}&\multirow{3}*{\makecell[l]{Params used\\for fusion (M)}}\\
&&{\makecell[c]{Mean IoU (\%)}} &{$\;\;$ \makecell[l]{Params\\in total (M)}} &{\makecell[c]{Mean IoU (\%)\\RGB / Depth / Ensemble}} &{$\;$ \makecell[l]{Params\\in total (M)}} & \\
\midrule
 RGB & Uni-modal&45.5&118.1&45.5 / $\;\;\,$-$\;\;\,$ / $\;\;$-$\;\;\;$&118.1&-\\
 Depth & Uni-modal&35.8&118.1&$\;$-$\;\;$ / 35.8 / $\;\;$-$\;$&118.1&-\\
\midrule
\multirow{15}{*}{RGB-D}
& Concat (early)&47.2&120.1&47.0 / 37.5 / 47.6&118.8&0.6\\
& Concat (middle)&46.7&147.7&46.6 / 37.0 / 47.4&120.3&2.1\\
& Concat (late)&46.3&169.0&46.3 / 37.2 / 46.9&126.6&8.4\\
& Concat (all-stage)&47.5&171.7&47.8 / 36.9 / 48.3&129.4&11.2\\
\cmidrule(r){2-7}
& Align (early)&46.4&238.8&46.3 / 35.8 / 46.7&120.8&2.6\\
& Align (middle)&47.9&246.7&47.7 / 36.0 / 48.1&128.7&10.5\\
& Align (late)&47.6&278.1&47.3 / 35.4 / 47.6&160.1&41.9\\
& Align (all-stage)&46.8&291.9&46.6 / 35.5 / 47.0&173.9&55.7\\
\cmidrule(r){2-7}
& Self-att. (early)&47.8&124.9&47.7 / 38.3 / 48.2&123.6&5.4\\
& Self-att. (middle)&48.3&166.9&48.0 / 38.1 / 48.7&139.4&21.2\\
& Self-att. (late)&47.5&245.5&47.6 / 38.1 / 48.3&203.2&84.9\\
& Self-att. (all-stage)&48.7&272.3&48.5 / 37.7 / 49.1&231.0&112.8\\
\cmidrule(r){2-7}
& Ours&-&-&\textbf{49.7} / \textbf{45.1} / \textbf{51.1}&\textbf{118.2}&\textbf{0.0}\\
\bottomrule
\end{tabular}}

\label{tabs:our_implementation}
\vskip -0.1 in
\end{table}

\begin{table}[t]
\centering

\caption{Comparison with SOTA methods on semantic segmentation.}
\vspace{-0.3em}
\resizebox{136mm}{!}{
\begin{tabular}{lll|ccc|ccc}
\toprule
\multirow{3}*{Modality}&\multirow{3}*{Approach}&\multirow{3}*{\makecell[l]{Backbone \\Network}}&\multicolumn{3}{c|}{NYUDv2}&\multicolumn{3}{c}{SUN RGB-D}\\
&&&\makecell[l]{Pixel Acc.\\(\%)} & \makecell[l]{Mean Acc.\\(\%)} & \makecell[l]{Mean IoU\\(\%)}& \makecell[l]{Pixel Acc.\\(\%)} & \makecell[l]{Mean Acc.\\(\%)} & \makecell[l]{Mean IoU\\(\%)} \\

\midrule
\multirow{3}*{RGB} 
&FCN-32s  \cite{journals/corr/LongSD14}&VGG16&60.0&42.2&29.2&68.4&41.1&29.0\\

&RefineNet \cite{lin2019refinenet}&ResNet101&73.8&58.8&46.4&80.8&57.3&46.3\\
&RefineNet \cite{lin2019refinenet}&ResNet152&74.4&59.6&47.6&81.1&57.7&47.0\\
\midrule
\multirow{13.5}*{RGB-D} 
& FuseNet  \cite{conf/accv/HazirbasMDC16}&VGG16 &68.1&50.4&37.9&76.3&48.3&37.3\\
&ACNet \cite{conf/icip/HuYFW19}&ResNet50&-&-&48.3&-&-&48.1\\
&SSMA \cite{journals/ijcv/ValadaMB20}&ResNet50&75.2&60.5&48.7&81.0&58.1&45.7\\
&SSMA \cite{journals/ijcv/ValadaMB20} $\dag$&ResNet101&75.8&62.3&49.6&81.6&60.4&47.9\\
&CBN~\cite{conf/nips/VriesSMLPC17} $\dag$ & ResNet101&75.5&61.2&48.9&81.5&59.8&47.4\\
&3DGNN \cite{conf/iccv/QiLJFU17}&ResNet101&-&-&-&-&57.0&45.9\\

&SCN \cite{journals/tcyb/LinZJLH20} &ResNet152&-&-&49.6&-&-&50.7\\

& CFN \cite{conf/iccv/LinCCHH17} & ResNet152&-&-&47.7&-&-&48.1\\
&RDFNet \cite{conf/iccv/LeePH17} &ResNet101&75.6&62.2&49.1&80.9&59.6&47.2\\
&RDFNet \cite{conf/iccv/LeePH17} &ResNet152&76.0&62.8&50.1&81.5&60.1&47.7\\
\cmidrule(r){2-9}
&Ours-RefineNet (single-scale) &ResNet101&76.2&62.8&51.1&82.0&60.9&49.6\\
&Ours-RefineNet
&ResNet101&{77.2}&{63.7}&{51.7}&{82.8}&{61.9}&{50.2}\\
&Ours-RefineNet &ResNet152&{77.4}&{64.8}&{52.2}&{83.2}&{62.5}&{50.8}\\
&Ours-PSPNet &ResNet152&\textbf{77.7}&\textbf{65.0}&\textbf{52.5}&\textbf{83.5}&\textbf{63.2}&\textbf{51.1}\\

\bottomrule
\end{tabular}}
\label{tabs:seg_results}

\begin{tablenotes}
  \footnotesize
  \item[1] $\dag$ indicates our implemented results.
\end{tablenotes}

\vskip -0.2 in
\end{table}

\subsection{Semantic Segmentation}
\textbf{Datasets.} 
We evaluate our method on two public datasets NYUDv2~\cite{conf/eccv/SilbermanHKF12} and SUN RGB-D~\cite{conf/cvpr/SongLX15}, which consider RGB and depth as input. Regarding NYUDv2, we follow the standard settings and adopt the split of 795 images for training and 654 for testing, with predicting standard 40 classes~\cite{conf/cvpr/GuptaAM13}. SUN RGB-D is one of the most challenging large-scale benchmarks towards indoor semantic segmentation, containing 10,335 RGB-D images of 37 semantic classes. We use the public train-test split (5,285 vs 5,050).

\textbf{Implementation.} We consider RefineNet~\cite{lin2019refinenet}/PSPNet~\cite{conf/cvpr/ZhaoSQWJ17} as our segmentation framework whose backbone is implemented by ResNet~\cite{conf/cvpr/HeZRS16} pretrained from ImageNet dataset \cite{journals/ijcv/RussakovskyDSKS15}. 
The initial learning rates are set to $5\times10^{-4}$ and $3\times10^{-3}$ for the encoder and decoder, respectively, both of which are reduced to their halves every 100/150 epochs (total epochs 300/450) on NYUDv2 with ResNet101/ResNet152 and every 20 epochs (total epochs 60) on SUN RGB-D. The mini-batch size, momentum and weight decay are selected as 6, 0.9, and $10^{-5}$, respectively, on both datasets. 
We set $\lambda=5\times10^{-3}$ in Equation~\ref{eq:cen} and the threshold to $\theta=2\times10^{-2}$ in Equation~\ref{eq:exchange-bn}.
Unless otherwise specified, we adopt the multi-scale strategy~\cite{conf/iccv/LeePH17,lin2019refinenet} for test. 
We employ the Mean IoU along with Pixel Accuracy and Mean Accuracy as evaluation metrics following~\cite{lin2019refinenet}. Full implementation details are referred to our appendix.

\textbf{The validity of each proposed component.}
Table \ref{tabs:component} summarizes the results of different variants of CEN on NYUDv2. We have the following observations:
\textbf{1.} Compared to the unshared baseline, sharing the convolutional parameters greatly boosts the performance particularly on the Depth modality (35.8 vs 38.4). Yet, the performance will encounter a clear drop if we additionally share the BN layers. This observation is consistent with our analyses in~\textsection~\ref{subsec:net_sharing} due to the different roles of convolutional filters and BN parameters.
\textbf{2.} After carrying out directed channel exchanging under the $\ell_1$ regulation, our model gains a huge improvement on both modalities, \emph{i.e.} from 46.0 to 49.7 on RGB, and from 38.1 to 45.1 on Depth, and finally increases the ensemble Mean IoU from 47.6 to 51.1. It thus verifies the effectiveness of our proposed mechanism on this task.
\textbf{3.}
Note that the channel exchanging is only available on a certain portion of each layer (\emph{i.e.} the half of the channels in the two-modal case). When we remove this constraint and allow all channels to be exchanged by Equation~\ref{eq:exchange-bn}, the accuracy decreases, which we conjecture is owing to the detriment by impeding modal-specific propagation, if all channels are engaged in cross-modal fusion.

To further explain why channel exchanging works, Figure \ref{featuremap5} displays the feature maps of RGB and Depth, where we find that the RGB channel with non-zero scaling factor mainly characterizes the texture, while the Depth channel with non-zero factor focuses more on the boundary; in this sense, performing channel exchanging can better combine the complementary properties of both modalities.

\textbf{Comparison with other fusion baselines.}
Table~\ref{tabs:our_implementation} reports the comparison of our CEN with two aggregation-based methods: concatenation~\cite{conf/cvpr/ZengTHYSCW19} and self-attention~\cite{journals/ijcv/ValadaMB20}, and one alignment-based approach~\cite{conf/eccv/WangWTSW16},  using the same backbone. All baselines are implemented with the early, middle, late, and all stage fusion. Besides, for a more fair comparison, all baselines are further  conducted under the same setting (except channel exchanging) with ours, namely, sharing convolutions with private BNs, and preserving the propagation of all sub-networks. Full details are provided in the appendix. It demonstrates that, on both settings, our method always outperforms others by an average improvement more than 2\%. We also report the parameters used for fusion, \emph{e.g.} the aggregation weights of two modalities in concatenation. While self-attention (all-stage) attains the closest performance to us (49.1 vs 51.1), the parameters it used for fusion are considerable, whereas our fusion is parameter-free.

\textbf{Comparison with SOTAs.}
We contrast our method against a wide range of state-of-the-art methods. Their results are directly copied from previous papers if provided or re-implemented by us otherwise, with full specifications illustrated in the appendix. 
Table~\ref{tabs:seg_results} concludes that our method equipped with PSPNet (ResNet152) achieves new records remarkably superior to previous methods in terms of all metrics on both datasets. In particular, given the same backbone, our method are still much better than RDFNet~\cite{conf/iccv/LeePH17}. To isolate the contribution of RefineNet in our method, Table~\ref{tabs:seg_results} also provides the uni-modal results, where we observe a clear advantage of multimodal fusion.

\textbf{Additional ablation studies.}
In this part, we provide some additional experiments on NYUDv2, with RefineNet (ResNet101). Results are obtained with single-scale evaluation.
\textbf{1.} As $\ell_1$ enables the discovery of unnecessary channels and comes as a pre-condition of Theorem~\ref{th:theorem}, naively exchanging channels with a fixed portion (without using $\ell_1$ and threshold) could not reach good performance. For example, exchanging a fixed portion of 30\% channels only gets IoU 47.2. We also find by only exchanging 30\% channels at each down-sampling stage of the encoder, instead of every $3\times3$ convolutional layer throughout the encoder (like our CEN), the result becomes 48.6, which is much lower than our CEN (51.1). \textbf{2.} In Table~\ref{tabs:seg_results}, we provide results of our implemented CBN~\cite{conf/nips/VriesSMLPC17} by modulating the BN of depth conditional on RGB. The IoUs of CBN with unshared and shared convolutional parameters are 48.3 and 48.9, respectively.
\textbf{3.} Directly summing activations (discarding the 1st term in Equation~\ref{eq:exchange-bn}) results in IoU 48.1, which could reach 48.4 when summing with a learnt soft gate. \textbf{4.} If we replace the ensemble of expert with a concat-fusion block, the result will slightly reduce from 51.1 to 50.8. \textbf{5.} Besides, we try to exchange channels randomly like ShuffleNet or directly discard unimportant channels without channel exchanging, the IoUs of which are 46.8 and 47.5, respectively. All above ablations support the optimal design of our architecture.

\begin{figure}[t]
\centering
\hskip -0.1 in
\includegraphics[scale=0.22]{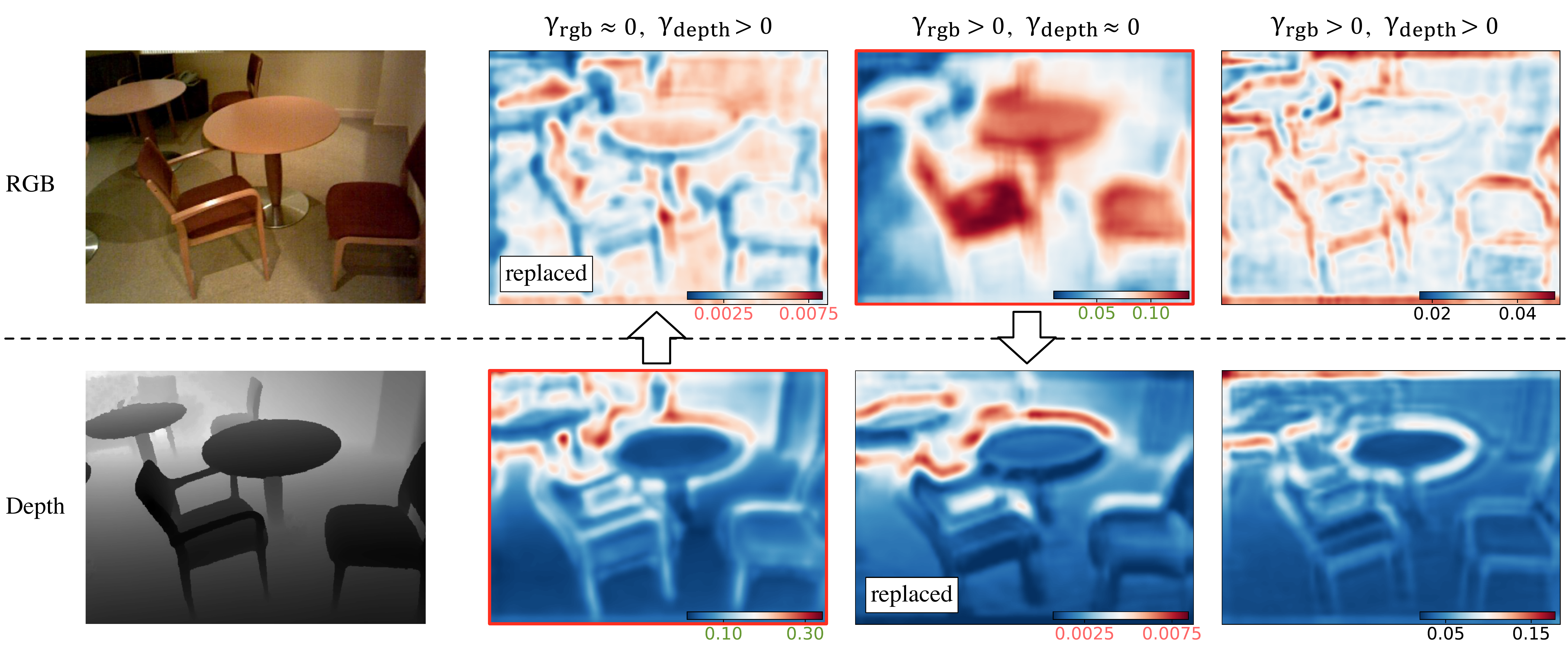}
\vskip -0.1 in
\caption{ Visualization of the averaged feature maps for RGB and Depth. From left to right: the input images, the channels of $(\gamma_{rgb}\approx 0, \gamma_{depth}> 0)$, $(\gamma_{rgb}> 0, \gamma_{depth}\approx 0)$, and $(\gamma_{rgb}> 0, \gamma_{depth}> 0)$.}
\label{featuremap5}

\vskip -0.06 in
\end{figure}

\begin{table}[t]
\centering

\caption{Comparison on image-to-image translation. Evaluation metrics are FID/KID ($\times 10^{-2}$). Lower values indicate better performance.}
\vspace{-0.3em}
\resizebox{136mm}{!}{
\begin{tabular}{p{2.2cm}|p{1.9cm}<{\centering} |p{1.3cm}p{2cm}<{\centering}p{2cm}<{\centering}p{2cm}<{\centering}p{2cm}<{\centering}p{2cm}<{\centering}}

\toprule[1pt]
$\;$Modality&Ours&Baseline&Early&Middle&Late&All-layer\\
\midrule

\multirow{4}{*}{\makecell[l]{Shade+Texture\\\;\;$\to$RGB}} 
&\multirow{4}{*}{\textbf{62.63} / \textbf{1.65}}
& Concat&{87.46 / 3.64}&{95.16 / 4.67}&{122.47 / 6.56}&{78.82 / 3.13}\\
&& Average&{93.72 / 4.22}&{93.91 / 4.27}&{126.74 / 7.10}&{80.64 / 3.24}\\
&&Align&{99.68 / 4.93}&{95.52 / 4.75}&{\;\;98.33 / 4.70}&{92.30 / 4.20}\\
&&Self-att.&{83.60 / 3.38}&{90.79 / 3.92}&{105.62 / 5.42}&{73.87 / 2.46}\\
\midrule
\multirow{4}{*}{\makecell[l]{Depth+Normal\\\;\;$\to$RGB}} 
&\multirow{4}{*}{\textbf{84.33} / \textbf{2.70}}
& Concat&{105.17 / 5.15}&{100.29 / 3.37}&{116.51 / 5.74}&{\;\;99.08 / 4.28}\\
&& Average&{109.25 / 5.50}&{104.95 / 4.98}&{122.42 / 6.76}&{\;\;99.63 / 4.41}\\
&&Align&{111.65 / 5.53}&{108.92 / 5.26}&{105.85 / 4.98}&{105.03 / 4.91}\\
&&Self-att.&{100.70 / 4.47}&{\;\;98.63 / 4.35}&{108.02 / 5.09}&{\;\;96.73 / 3.95}\\

\bottomrule[1pt]

\end{tabular}}
\label{tabs:translation}
\vskip -0.03 in
\end{table}







\begin{table}[t!]
\centering

\caption{Multimodal fusion on image translation (to RGB) with modalities from 1 to 4.}
\vspace{-0.3em}
\resizebox{136mm}{!}{
\begin{tabular}{l|cccc|ccc}

\toprule[1pt]
Modality&Depth&Normal&Texture&Shade&\makecell[l]{Depth+Normal}&\makecell[l]{Depth+Normal\\+Texture}&\makecell[l]{Depth+Normal\\+Texture+Shade}\\
\midrule
FID&113.91 &108.20 &97.51 &100.96   &84.33&60.90&57.19\\
KID ($\times 10^{-2}$)& 5.68&5.42&4.82&5.17 & 2.70& 1.56&1.33\\

\bottomrule[1pt]

\end{tabular}}
\label{tabs:multimodal}
\vskip -0.12 in
\end{table}

\subsection{Image-to-Image Translation}
\textbf{Datasets.} We adopt Taskonomy~\cite{conf/cvpr/ZamirSSGMS18}, a dataset with 4 million images of indoor scenes of about 600 buildings. Each image in Taskonomy has more than 10 multimodal representations, including depth (euclidean/zbuffer), shade, normal, texture, edge, principal curvature, etc. For efficiency, we sample 1,000 high-quality multimodal images for training, and 500 for validation.

\textbf{Implementation.} Following Pix2pix~\cite{conf/cvpr/IsolaZZE17}, we adopt the U-Net-256 structure for image translation with the consistent setups with~\cite{conf/cvpr/IsolaZZE17}. The BN computations are replaced with Instance Normalization layers (INs), and our method (Equation~\ref{eq:exchange-bn}) is still applicable. We adopt individual INs in the encoder, and share all other parameters including INs in the decoder. We set $\lambda$ to $10^{-3}$ for sparsity constraints and the threshold $\theta$ to $10^{-2}$. We adopt FID~\cite{conf/nips/HeuselRUNH17} and KID~\cite{conf/iclr/BinkowskiSAG18} as evaluation metrics, which will be introduced in our appendix.

\textbf{Comparison with other fusion baselines.}
In Table~\ref{tabs:translation}, we evaluate the performance on two specific translation cases, \emph{i.e.} Shade+Texture$\to$RGB and Depth+Normal$\to$RGB, with more examples included in the appendix. In addition to the three baselines used in semantic segmentation (Concat, Self-attention, Align),  we conduct an extra aggregation-based method by using the average operation. All baselines perform fusion under 4 different kinds of strategies: early (at the 1st conv-layer), middle (the 4th conv-layer), late (the 8th conv-layer), and all-layer fusion. As shown in Table~\ref{tabs:translation}, our method yields much lower FID/KID than others, which supports the benefit of our proposed idea once again.  

\textbf{Considering more modalities.}
We now test whether our method is applicable to the case with more than 2 modalities. For this purpose, Table~\ref{tabs:multimodal} presents the results of image translation to RGB by inputting from 1 to 4 modalities of Depth, Normal, Texture, and Shade. It is observed that increasing the number of modalities improves the performance consistently, suggesting much potential of applying our method towards various cases.

\section{Conclusion}
\vskip -0.03 in
In this work, we propose Channel-Exchanging-Network (CEN), a novel framework for deep multimodal fusion, which differs greatly with existing aggregation-based and alignment-based multimodal fusion. The motivation behind is to boost inter-modal fusion while simultaneously keeping sufficient intra-modal processing. The channel exchanging is self-guided by channel importance measured by individual BNs, making our framework self-adaptive and compact. Extensive evaluations verify the effectiveness of our method.

\section*{Acknowledgement}
This work is jointly funded by National Natural Science Foundation of China and German Research Foundation (NSFC 61621136008/DFG TRR-169) in project ``Crossmodal Learning'' \textrm{II}, Tencent AI Lab Rhino-Bird Visiting Scholars Program (VS202006), and China Postdoctoral
Science Foundation (Grant No.2020M670337).

\section*{Broader Impact}
This research enables fusing complementary information from different modalities effectively, which helps improve performance for autonomous vehicles and indoor manipulation robots, also making them more robust to environmental conditions, \emph{e.g.} light, weather. Besides, instead of carefully designing hierarchical fusion strategies in existing methods, a global criterion is applied in our work for guiding multimodal fusion, which allows easier model deployment for practical applications. A drawback of bringing deep neural networks into multimodal fusion is its insufficient interpretability.

{\small
\bibliographystyle{splncs04.bst}
\bibliography{egbib}
}


\clearpage

\section*{\LARGE Appendix}
\appendix

\section{Proofs}
\setcounter{theorem}{0}
\begin{theorem}
Suppose $\{\gamma_{m,l,c}\}_{m,l,c}$ are the BN scaling factors of any multimodal fusion network (without channel exchanging) optimized by Equation~\ref{eq:cen}. Then the probability of $\gamma_{m,l,c}$ being attracted to $\gamma_{m,l,c}=0$  during training (\emph{a.k.a.} $\gamma_{m,l,c}=0$ is the local minimum) is equal to $2\Phi(\lambda|\frac{\partial L}{\partial \bm{x}'_{m,l,c}}|^{-1})-1$, where $\Phi$ derives the cumulative probability of standard Gaussian.
\end{theorem}
\begin{proof}
The proof is straightforward, since the gradient of $L$ \emph{w.r.t.} $\gamma_{m,l,c}$ is $\frac{\partial L}{\partial \bm{x}'_{m,l,c}}\frac{\bm{x}_{m,l,c}-\mu_{m,l,c}}{\sqrt{\sigma^2_{m,l,c}+\epsilon}}+\lambda$ when $\gamma_{m,l,c}>0$, or $\frac{\partial L}{\partial \bm{x}'_{m,l,c}}\frac{\bm{x}_{m,l,c}-\mu_{m,l,c}}{\sqrt{\sigma^2_{m,l,c}+\epsilon}}-\lambda$ when $\gamma_{m,l,c}<0$\footnote{Here, we denote $\frac{\partial L}{\partial \bm{x}'_{m,l,c}}\frac{\bm{x}_{m,l,c}-\mu_{m,l,c}}{\sqrt{\sigma^2_{m,l,c}+\epsilon}}=\sum_{(i,j)=1}^{(H,W)}\frac{\partial L}{\partial \bm{x}'_{m,l,c}}(i,j)\frac{\bm{x}_{m,l,c}(i,j)-\mu_{m,l,c}}{\sqrt{\sigma^2_{m,l,c}+\epsilon}}$ for alleviation, where $i,j$ range over each pixel in $\bm{x}'_{m,l,c}$ or $\bm{x}_{m,l,c}$.}, according to the BN definition in Equation~\ref{eq:bn} and the $\ell_1$ norm in Equation~\ref{eq:cen}.
Staying around $\gamma_{m,l,c}=0$ during training implies that $\frac{\partial L}{\partial \bm{x}'_{m,l,c}}\frac{\bm{x}_{m,l,c}-\mu_{m,l,c}}{\sqrt{\sigma^2_{m,l,c}+\epsilon}}+\lambda>0$ as well as $\frac{\partial L}{\partial \bm{x}'_{m,l,c}}\frac{\bm{x}_{m,l,c}-\mu_{m,l,c}}{\sqrt{\sigma^2_{m,l,c}+\epsilon}}-\lambda<0$, the probability of which is $2\Phi(\lambda|\frac{\partial L}{\partial \bm{x}'_{m,l,c}}|^{-1})-1$ given that the quantity $\frac{\bm{x}_{m,l,c}-\mu_{m,l,c}}{\sqrt{\sigma^2_{m,l,c}+\epsilon}}$ can be considered as a random variable of standard Gaussian according to the central limit theorem.
\end{proof}

\begin{figure}[h!]
\centering
\vspace{-1em}
\includegraphics[scale=0.4]{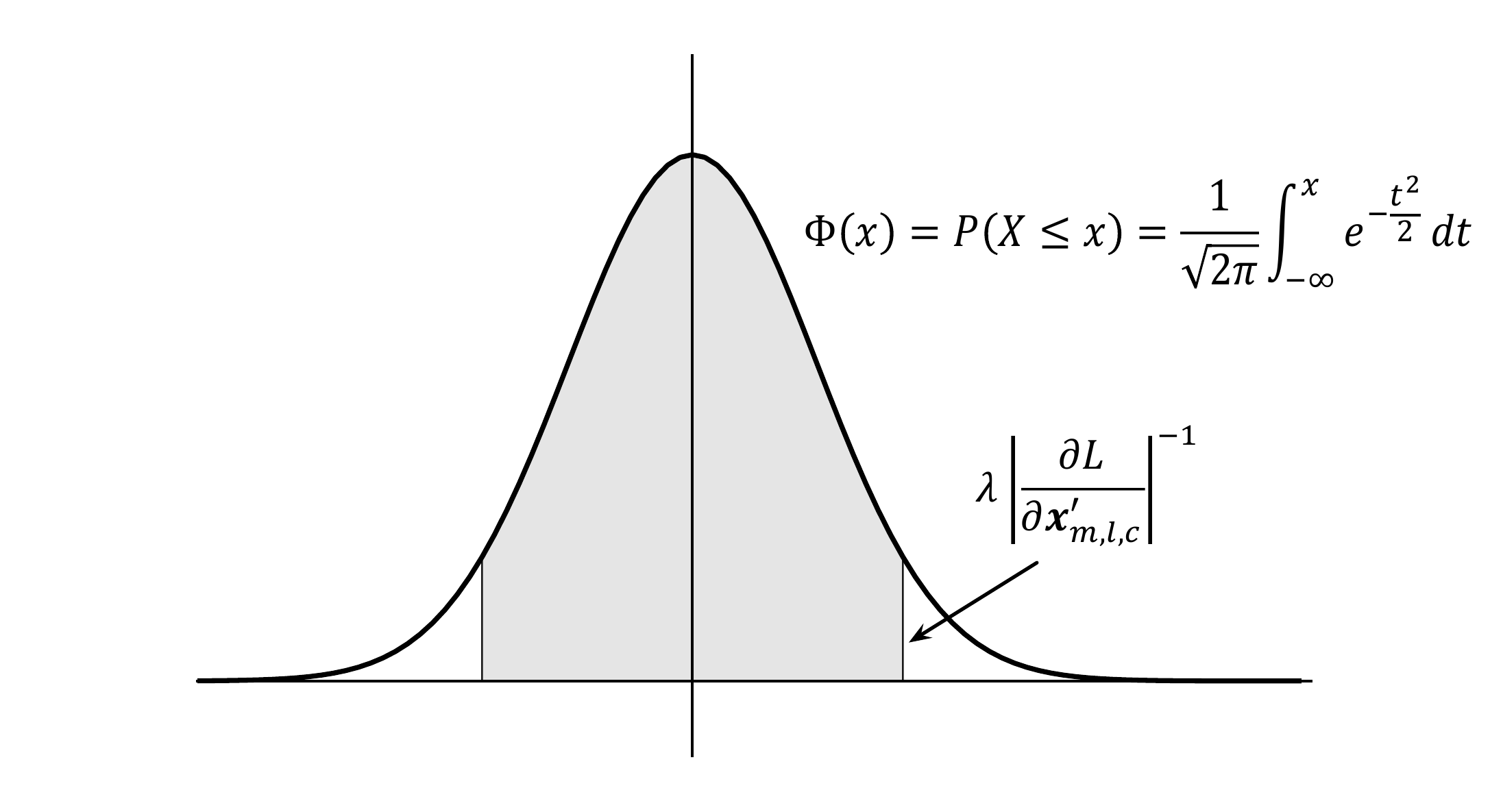}
\vspace{-1em}
\caption{Illustration of the conclusion by Theorem 1.}
\label{pic:normal_distribution}
\vspace{1em}
\end{figure}

\begin{figure}[h!]
\centering
\includegraphics[width=\linewidth]{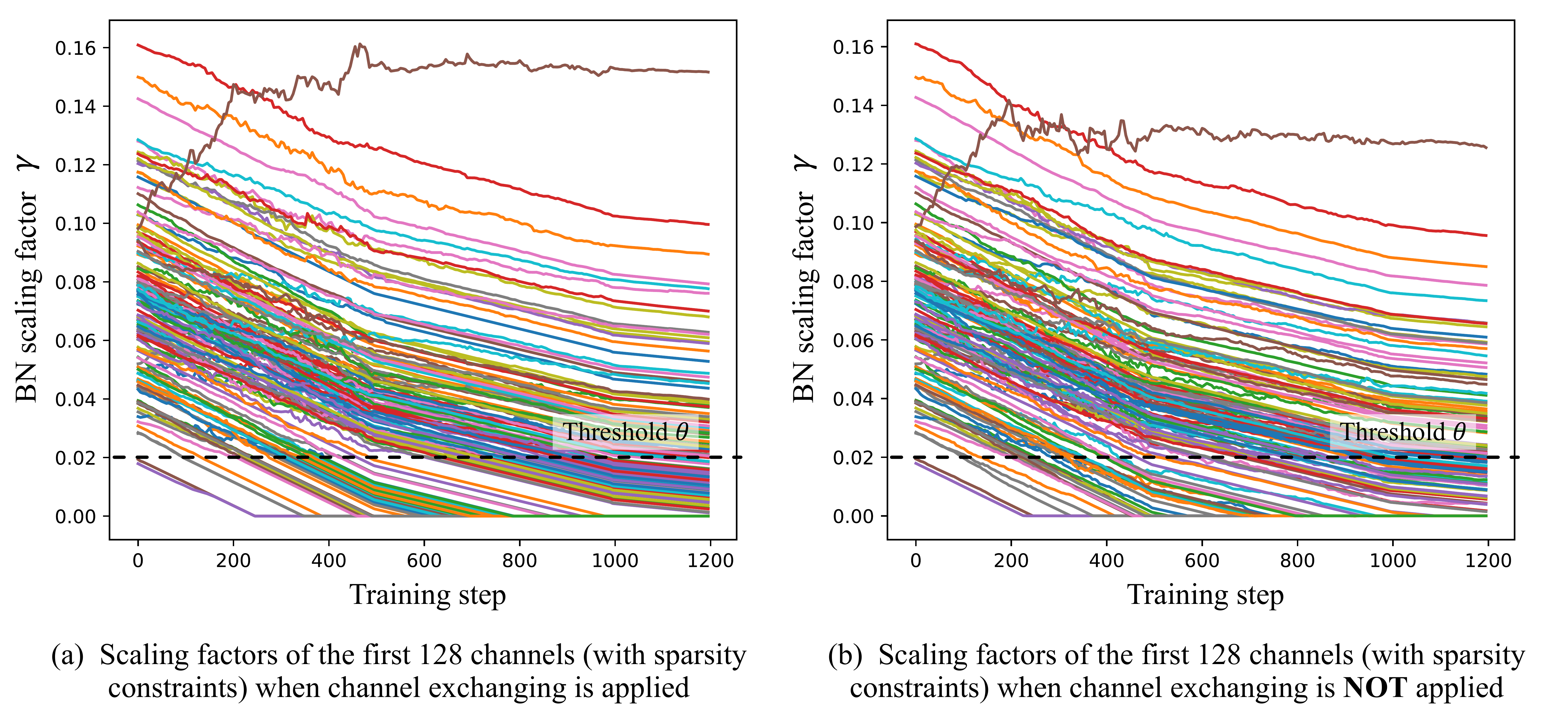}
\caption{We plot BN scaling factors with sparsity constraints vs training steps. We observe that whether using channel exchanging or not, $\gamma$ that closes to zero can hardly recover, which verifies our conjecture in Theorem 1. The experiment is conducted on NYUDv2 with RefineNet (ResNet101). We choose the 8th layer of convolutional layers that have $3\times3$ kernels, and there are totally 256 channels in this layer. Regarding the RGB modality, the sparsity constraints to BN scaling factors are applied for the first 128 channels.}
\label{pic:scaling_factor}
\end{figure}

\setcounter{theorem}{0}
\begin{corollary}
\vspace{-0.3em}
If the minimal of Equation~\ref{eq:cen} implies $\gamma_{m,l,c}=0$, then the channel exchanging by Equation~\ref{eq:exchange-bn} (assumed no crossmodal parameter sharing) will only decrease the training loss, \emph{i.e.} $\min_{f'_{1:M}}L\leq\min_{f_{1:M}}L$, given the sufficiently expressive $f'_{1:M}$ and $f_{1:M}$ which denote the cases with and without channel exchanging, respectively.
\end{corollary}
\begin{proof}
We only need to prove for any $f_{1:M}$, we can design a specific $f'_{1:M}$ that shares the same output as $f_{1:M}$ if $\gamma_{m,l,c}=0$.
\begin{itemize}
    \item In $f_{1:M}$, the BN layer is followed by a ReLU function and a convolutional layer. We suppose the following convolutional weight for the $c$-th input channel $\bm{x}'_{m,l,c}$ is $\bm{W}_{m,l+1,c}$ and the bias is $b_{m,l+1}$. Thus, the quantity related to $\bm{x}'_{m,l,c}$ in the $(l+1)$-th layer is $\bm{W}_{m,l+1,c}\otimes\sigma(\bm{x}'_{m,l,c})+
    b_{m,l+1}$, where $\otimes$ denotes the convolution operation and $\sigma$ is the ReLU function. Since $\gamma_{m,l,c}=0$, this term can be translated as $\bm{W}_{m,l+1,c}\otimes\sigma(\beta_{m,l,c})+
    b_{m,l+1}$, which is a constant feature map. 
    \item As for $f'_{1:M}$, we apply the similar denotations, and attain the term related to $\bm{x}'_{m,l,c}$ in the $(l+1)$-th layer as $\bm{W}'_{m,l+1,c}\otimes\sigma(\bm{x}'_{m,l,c})+
    b'_{m,l+1}$. 
\end{itemize}
By setting $b'_{m,l+1}=\bm{W}_{m,l+1,c}\otimes\sigma(\beta_{m,l,c})+b_{m,l+1}$ and $\bm{W}'_{m,l+1,c}=0$, we will always have $f'_{1:M}=f_{1:M}$, which concludes the proof.
\end{proof}

In Figure \ref{pic:normal_distribution}, we provide an illustration of the conclusion by Theorem 1. In Figure \ref{pic:scaling_factor}, we provide experimental results to verify our conjecture in Theorem 1, \emph{i.e.} when the scaling factor of one channel is equal to zero at a certain training step, this channel will almost become redundant during later training process.

In summary, we know that $\ell_1$ makes the parameters sparse, but it can not tell if each sparse parameter will keep small in training considering the gradient in Equation~\ref{eq:cen}. Conditional on BN, Theorem 1 proves that $\gamma=0$ is attractive. Corollary 1 states that $f'$ is more expressive than $f$ when $\gamma=0$, and thus the optimal $f'$ always outputs no higher loss, which, yet, is not true for arbitrary $f'$ (\emph{e.g.} $f'=10^6$). Besides, as stated, Corollary 1 holds upon unshared convolutional parameters, and is consistent with Table \ref{tabs:supp_component} in the unshared scenario (full-channel: 49.1 vs half-channel: 48.5), although full-channel exchanging is worse under the sharing setting.

\section{Implementation Details}
\label{sec:details}

In our experiments, we adopt ResNet101, ResNet152 for semantic segmentation and U-Net-256 for image-to-image translation. Regarding both ResNet structures, we apply sparsity constraints on Batch-Normalization (BN) scaling factors \emph{w.r.t.} each convolutional layer (conv) with $3\times3$ kernels. These scaling factors further guide the channel exchanging process that exchanges a portion of feature maps after BN. For the conv layer with $7\times7$ kernels at the beginning of ResNet, and all other conv layers with $1\times1$ kernels, we do not apply sparsity constraints or channel exchanging. For U-Net, we apply sparsity constraints on Instance-Normalization (IN) scaling factors \emph{w.r.t.} all conv layers (eight layers in total) in the encoder of the generator, and each is followed by channel exchanging.

We mainly use three multimodal fusion baselines in our paper, including concatenation, alignment and self-attention. Regarding the concatenation method, we stack multimodal feature maps along the channel, and then add a $1\times1$ convolutional layer to reduce the number of channels back to the original number. The alignment fusion method is a re-implementation of \cite{conf/eccv/WangWTSW16}, and we follow its default settings for hyper-parameter, \emph{e.g.} using 11 kernel functions for the multiple kernel Maximum Mean Discrepancy. The self-attention method is a re-implementation of the SSMA block proposed in \cite{journals/ijcv/ValadaMB20}, where we also follow the default settings, \emph{e.g.} setting the channel reduction ratio $\eta$ to 16.

In Table~\ref{tabs:our_implementation}, we adopt early, middle, late and all-stage fusion for each baseline method. In ResNet101, there are four stages with 3, 4, 23, 3 blocks, respectively. The early fusion, middle fusion and late fusion refer to fusing after the 2nd stage, 3rd stage and 4th stage respectively. All-stage fusion refers to fusing after the four stages.

We use a NVIDIA Tesla V100 with 32GB for the experiments.

We now introduce the metrics used in our image-to-image translation task. In Table~\ref{tabs:translation}, we adopt the following evaluation metrics:

Fréchet-Inception-Distance (FID) proposed by \cite{conf/nips/HeuselRUNH17}, contrasts the statistics of generated samples against real samples. The FID fits a Gaussian distribution to the hidden activations of InceptionNet for each compared image set and then computes the Fréchet distance (also known as the Wasserstein-2 distance) between those Gaussians. Lower FID is better, corresponding to generated images more similar to the real. 

Kernel-Inception-Distance (KID) developed by~\cite{conf/iclr/BinkowskiSAG18}, is a metric similar to the FID but uses the squared Maximum-Mean-Discrepancy (MMD) between Inception representations with a polynomial kernel. Unlike FID, KID has a simple unbiased estimator, making it more reliable especially when there are much more inception features channels than image numbers. Lower KID indicates more visual similarity between real and generated images. Regarding our implementation of KID, the hidden representations are derived from the Inception-v3 pool3 layer.

\section{Additional Results}

We provide three more image translation cases in Table \ref{tabs:supp_translation}, including RGB+Shade$\to$Normal, RGB+Normal$\to$Shade and RGB+Edge$\to$Depth. For baseline methods, we adopt the same settings with Table~\ref{tabs:translation}, by adopting early (at the 1st conv-layer), middle (the 4th conv-layer), late (the 8th conv-layer) and all-layer fusion. We adopt MAE (L1 loss) and MSE (L2 loss) as evaluation metrics, and lower values indicate better performance. Our method yields lower MAE and MSE than baseline methods. 

\begin{table}[h]
\centering

\caption{Comparison on image-to-image translation. Evaluation metrics adopted  are MAE ($\times 10^{-1}$)/MSE ($\times 10^{-1}$). Lower values indicate better performance.}
\resizebox{136mm}{!}{
\begin{tabular}{p{2.2cm}|p{1.9cm}<{\centering} |p{1.3cm}p{2cm}<{\centering}p{2cm}<{\centering}p{2cm}<{\centering}p{2cm}<{\centering}}

\toprule[1pt]
$\;$Modality&Ours&Baseline&Early&Middle&Late&All-layer\\
\midrule

\multirow{4}{*}{\makecell[l]{RGB+Shade\\\;\;$\to$Normal}} 
&\multirow{4}{*}{\textbf{1.12} / \textbf{2.51}}& Concat&{1.33 / 2.83}&{1.22 / 2.65}&{1.39 / 2.88}&{1.34 / 2.85}\\
&& Average&{1.42 / 3.05}&{1.26 / 2.70}&{1.40 / 2.90}&{1.28 / 2.83}\\
&&Align&{1.45 / 3.11}&{1.39 / 2.93}&{1.28 / 2.76}&{1.52 / 3.25}\\
&&Self-att.&{1.30 / 2.82}&{1.18 / 2.59}&{1.42 / 2.91}&{1.26 / 2.76}\\
\midrule

\multirow{4}{*}{\makecell[l]{RGB+Normal\\\;\;$\to$Shade}} 
&\multirow{4}{*}{\textbf{1.10} / \textbf{1.72}}& Concat&{1.56 / 2.45}&{1.38 / 2.12}&{1.26 / 1.92}&{1.28 / 2.02}\\
&& Average&{1.46 / 2.29}&{1.28 / 2.04}&{1.51 / 2.39}&{1.23 / 1.86}\\
&&Align&{1.39 / 2.26}&{1.32 / 2.16}&{1.27 / 2.04}&{1.41 / 2.21}\\
&&Self-att.&{1.21 / 1.83}&{1.15 / 1.73}&{1.45 / 2.28}&{1.18 / 1.76}\\
\midrule

\multirow{4}{*}{\makecell[l]{RGB+Edge\\\;\;$\to$Depth}} 
&\multirow{4}{*}{\textbf{0.28} / \textbf{0.66}}& Concat&{0.34 / 0.75}&{0.32 / 0.74}&{0.38 / 0.79}&{0.33 / 0.75}\\
&& Average&{0.36 / 0.78}&{0.34 / 0.76}&{0.36 / 0.77}&{0.33 / 0.74}\\
&&Align&{0.44 / 0.89}&{0.39 / 0.82}&{0.42 / 0.86}&{0.44 / 0.90}\\
&&Self-att.&{0.30 / 0.71}&{0.33 / 0.73}&{0.34 / 0.75}&{0.30 / 0.70}\\

\bottomrule[1pt]

\end{tabular}}
\label{tabs:supp_translation}

\end{table}

\section{Results Visualization}
\label{sec:visual}

In Figure \ref{pic:seg_results} and Figure \ref{pic:city_seg_results}, we provide results visualization for the semantic segmentation task. We choose three baselines including concatenation (concat), alignment (align) and self-attention (self-att.). Among them, concatenation and self-attention methods adopt all-stage fusion, and the alignment method adopts middle fusion (fusion at the end of the 2nd ResNet stage). 

In Figure \ref{pic:trans_rgb}, Figure \ref{pic:trans_depth} and Figure \ref{pic:trans_normal_shade}, we provide results visualization for the image translation task. Regarding this task, concatenation and self-attention methods adopt all-layer fusion (fusion at all eight layers in the encoder), and the alignment method adopts middle fusion (fusion at the 4th layer). We adopt these settings in order to achieve high performance for each baseline method.

In the captions of these figures, we detail the prediction difference of different methods.

\section{Ablation Studies}

In Table \ref{tabs:supp_component}, we provide more cases as a supplement to Table \ref{tabs:component}. Specifically, we compare the results of channel exchaging when using shared/unshared conv parameters. According to these results, we believe our method is generally useful and channels are aligned to some extent even under the unshared setting.

In Table \ref{tabs:supp_multimodal}, we verify that sharing convolutional layers (convs) but using individual Instance-Normalization layers (INs) allows 2$\sim$4 modalities trained in a single network, achieving even better performance than training with individual networks. Again, if we further sharing INs, there will be an obvious performance drop. More detailed comparison is provided in Table \ref{tabs:supp_unshare_in}.

For the experiment Shade+Texture+Depth$\to$RGB with shared convs and unshared INs, 
in Figure \ref{pic:featuremap_onelayer}, we plot the proportion of IN scaling factors at the 7th conv layer in the encoder of U-Net. We compare the scaling factors when no sparsity constraints, sparsity constraints applied on all channels, and sparsity constraints applied on disjoint channels. In Figure \ref{pic:featuremap_all}, we further compare scaling factors on all conv layers. In Figure \ref{sec:params_sensitivity}, we provide sensitivity analysis for $\lambda$ and $\theta$.

\begin{table}[h]
\centering
\caption{Supplement to Table \ref{tabs:component} with more cases. Detailed results for different versions of our CEN on NYUDv2. All results are obtained with the backbone RefineNet (ResNet101) of single-scale evaluation for test. We observe that sharing convs (with unshared BNs) results in better performance for our method.}
\resizebox{115mm}{!}{
\begin{tabular}{llll|p{1.2cm}<{\centering}p{1.2cm}<{\centering}p{1.4cm}<{\centering}}

\toprule
\multirow{2}*{Convs}&\multirow{2}*{\makecell[l]{BNs}}&\multirow{2}*{$\ell_1$ Regulation}&\multirow{2}*{Exchange}&\multicolumn{3}{c}{Mean IoU (\%)}\\
&&&&RGB & Depth &Ensemble$\,$\\

\midrule
Unshared&Unshared&$\;\;\;\;\;\;\;\;\times$&$\;\;\;\;\;\,\times$&45.5&35.8&47.6\\
Shared & Shared&$\;\;\;\;\;\;\;\;\times$&$\;\;\;\;\;\,\times$&{43.7}&{35.5}&{45.2}\\
Shared&Unshared&$\;\;\;\;\;\;\;\;\times$&$\;\;\;\;\;\,\times$&46.2&38.4&48.0\\
\midrule
Unshared &Unshared  & Half-channel&$\;\;\;\;\;\,\times$&45.1&35.5&47.3\\
Unshared &Unshared  & Half-channel&$\;\;\;\;\;\,\checkmark$&{46.5}&{41.6}&{48.5}\\
Shared &Unshared  & Half-channel&$\;\;\;\;\;\,\times$&46.0&38.1&47.7\\
Shared &Unshared  & Half-channel &$\;\;\;\;\;\,\checkmark$&\textbf{49.7}&\textbf{45.1}&\textbf{51.1}\\
\midrule
Unshared &Unshared &All-channel&$\;\;\;\;\;\,\times$&44.6&35.3&46.6\\
Unshared &Unshared &All-channel&$\;\;\;\;\;\,\checkmark$&46.8&41.7&49.1\\
Shared &Unshared &All-channel&$\;\;\;\;\;\,\times$&46.1&37.9&47.5\\
Shared &Unshared &All-channel&$\;\;\;\;\;\,\checkmark$&48.6&{39.0}&49.8\\
\bottomrule
\end{tabular}}
\label{tabs:supp_component}
\end{table}

\begin{figure}[h]
\centering
\includegraphics[width=0.99\linewidth]{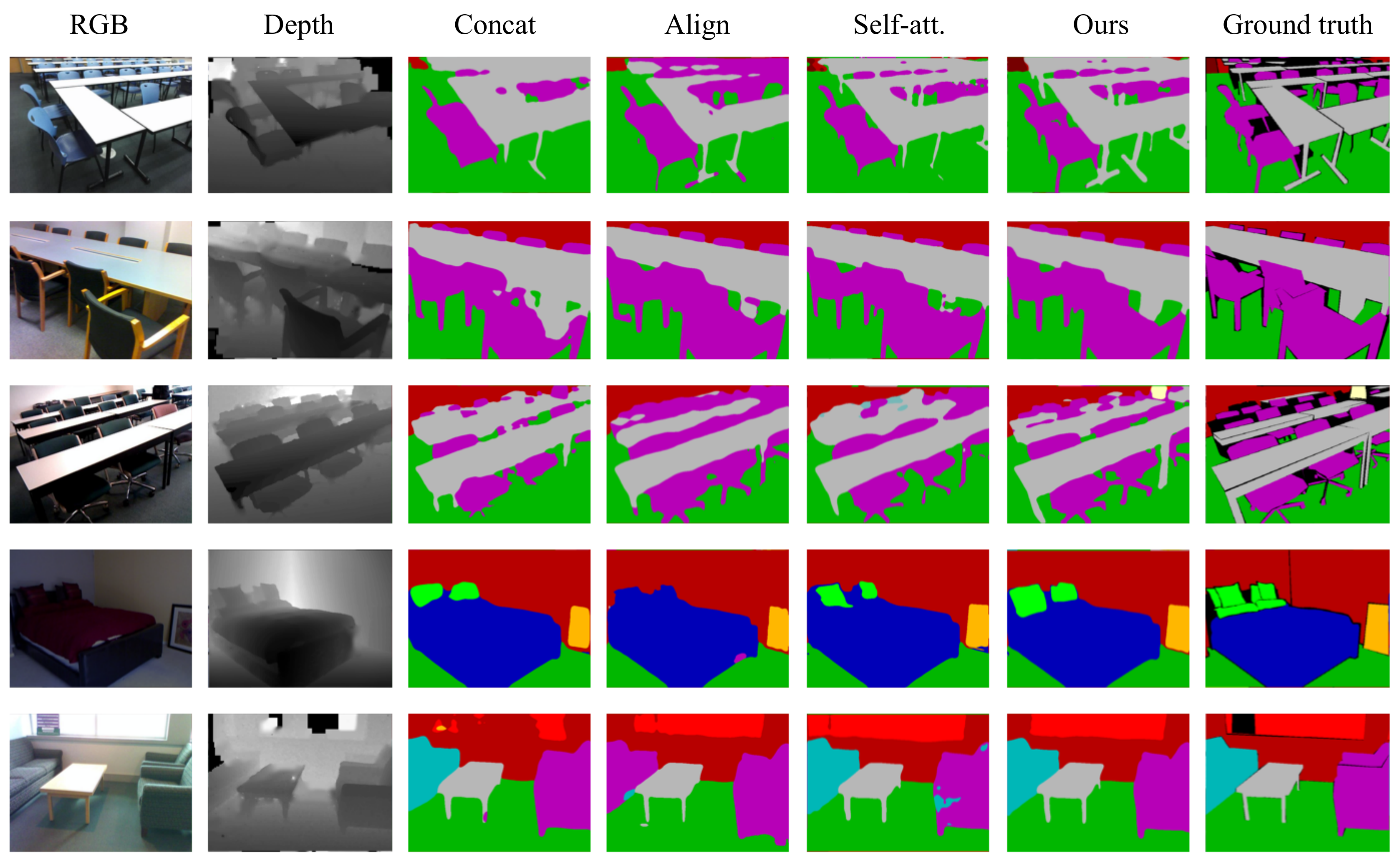}
\caption{Visualization results for semantic segmentation. Images are collected from NYUDv2 and SUN RGB-D dataset. All results are obtained with the backbone RefineNet (ResNet101) of single-scale evaluation for test. We choose tough images where a number of tables and chairs need to be predicted. Besides, we compare segmentation results on images with low/high light intensity. we observe that the concatenation method is more sensitive to noises of the depth input (see the window at bottom line). Both concatenation and self-attention methods are weak in predicting thin objects \emph{e.g.} table legs and chair legs. These objects are usually missed in the depth input, which may disturb the prediction results during fusion. Compared to baseline fusion methods, the prediction results of our method preserve more details, and are more robust to the light intensity.}
\label{pic:seg_results}
\end{figure}

\begin{figure}[h]
\centering
\includegraphics[width=0.98\linewidth]{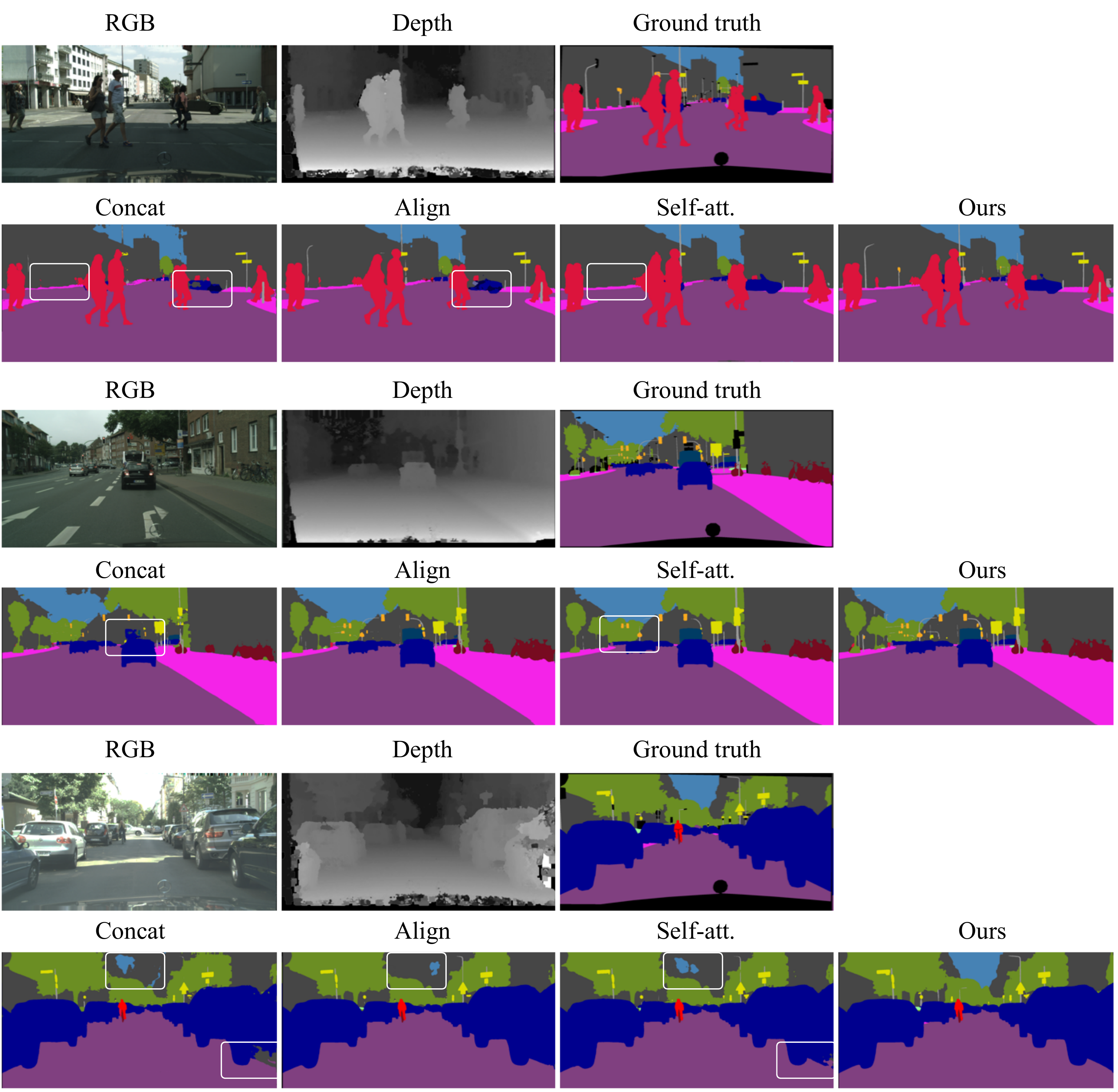}
\vspace{0.3em}
\caption{Visualization results for semantic segmentation on Cityscapes dataset \cite{conf/cvpr/CordtsORREBFRS16}. All results are obtained with the backbone PSPNet (ResNet101) of single-scale evaluation for test. Cityscapes is an outdoor dataset containing images from 27 cities in Germany and neighboring countries. The dataset contains 2,975 training, 500 validation and 1,525 test images. There are 20,000 additional coarse annotations provided by the dataset, which are not used for training in our experiments. For the baseline methods, we use white frames to highlight the regions with poor prediction results. We can observe that when the light intensity is high, the baseline methods are weak in capturing the boundary between the sky and buildings using the depth information. Besides, the concatenation and self-attention methods do not preserve fine-grained objects, \emph{e.g.} traffic signs, and are sensitive to noises of the depth input (see the rightmost vehicle in the first group). In contrast, the prediction of our method are better at these aforementioned aspects.}
\label{pic:city_seg_results}
\end{figure}

\begin{figure}[h]
\centering
\vspace{-1em}
\includegraphics[scale=0.465]{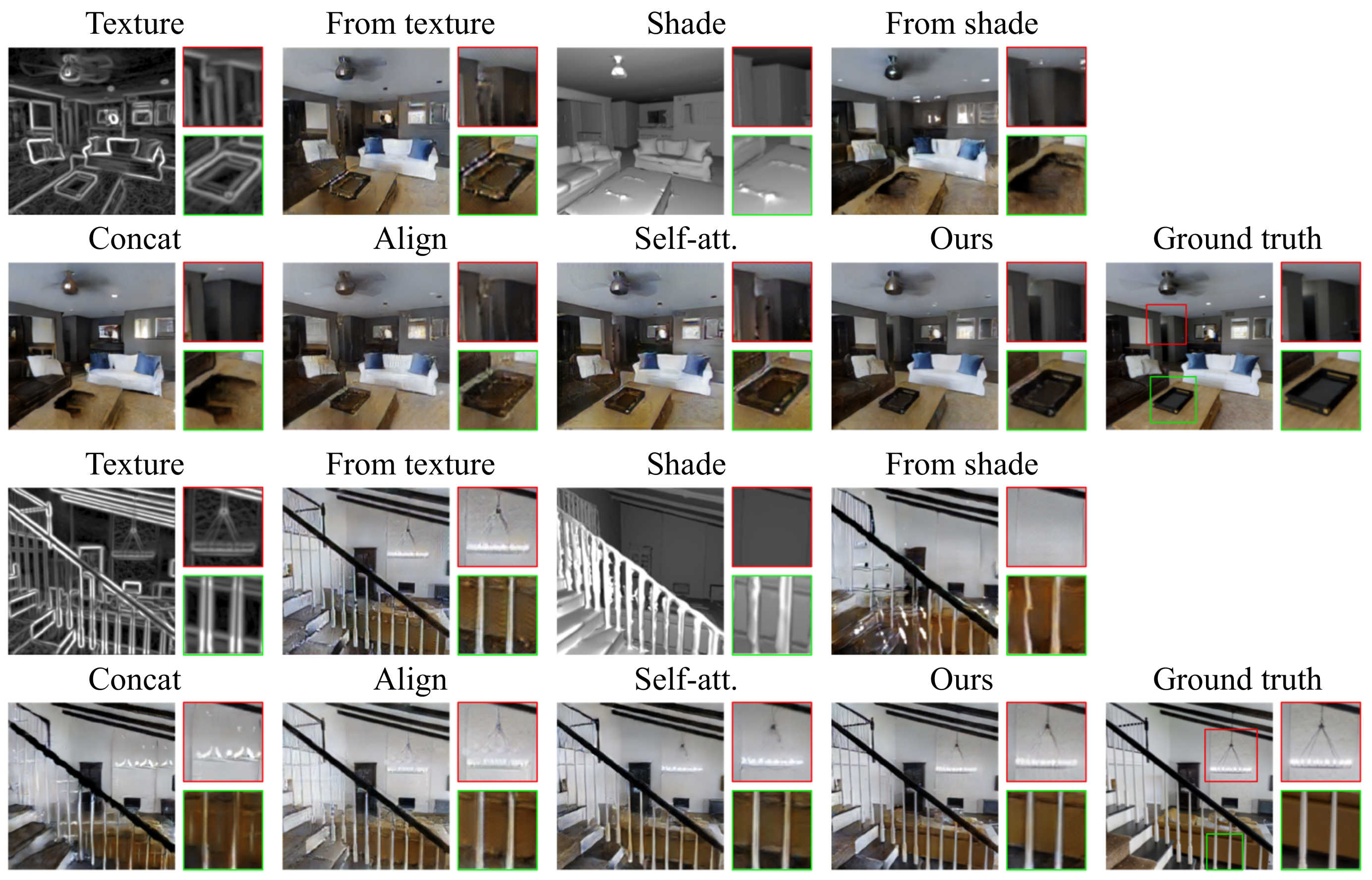}

\caption{Two groups of results comparison for image translation from Texture and Shade to RGB. We observe that the prediction solely predicted from the texture is vague at boundary lines, while the prediction from the shade misses some opponents, \emph{e.g.} the pendant lamp, and is weak in predicting handrails. When fusing the two modalities, the concatenation method is uncertain at the regions where both modalities have disagreements. Alignment and self-attention are still weak in combining both modalities at details. Our results are clear at boundaries and fine-grained details.}
\label{pic:trans_rgb}
\end{figure}

\begin{figure}[h]
\centering
\includegraphics[scale=0.465]{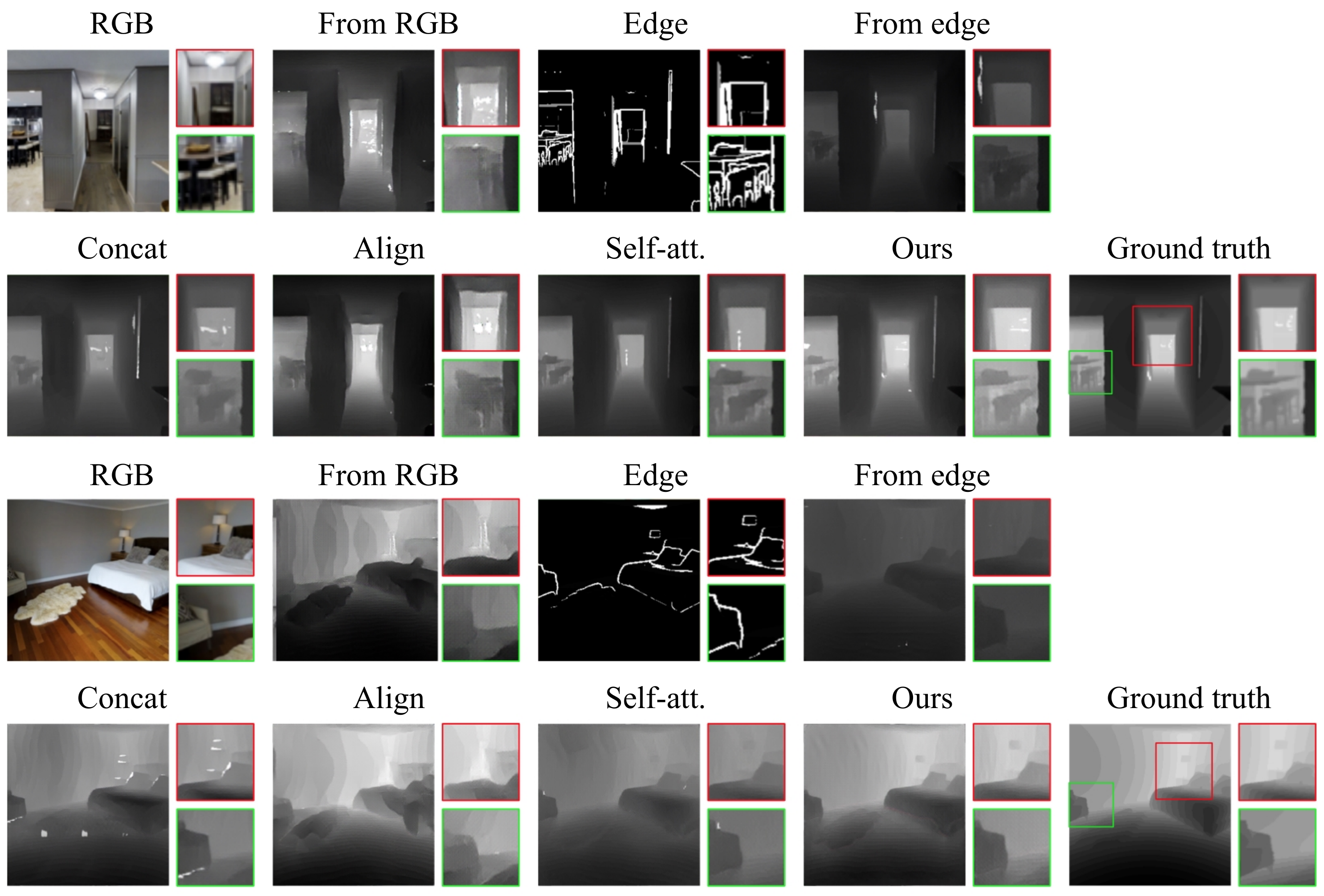}

\caption{Two groups of results comparison for image translation from RGB and Edge to Depth. It is straightforward to find the benefits of multimodal fusion in this figure. The depth predicted by RGB is good at predicting numerical values, but is weak in capturing boundaries, which results in curving walls. Oppositely, the depth predicted by the edge well captures boundaries, but is weak in determining numerical values. The alignment fusion method is still weak in capturing boundaries. Both concatenation and self-attention methods are able to combine the advantages of both modalities, but the numerical values are still obviously lower than the ground truth. Our prediction achieves better performance compared to baseline methods.}
\label{pic:trans_depth}
\vspace{-1em}
\end{figure}

\begin{figure}[h]
\centering
\includegraphics[scale=0.465]{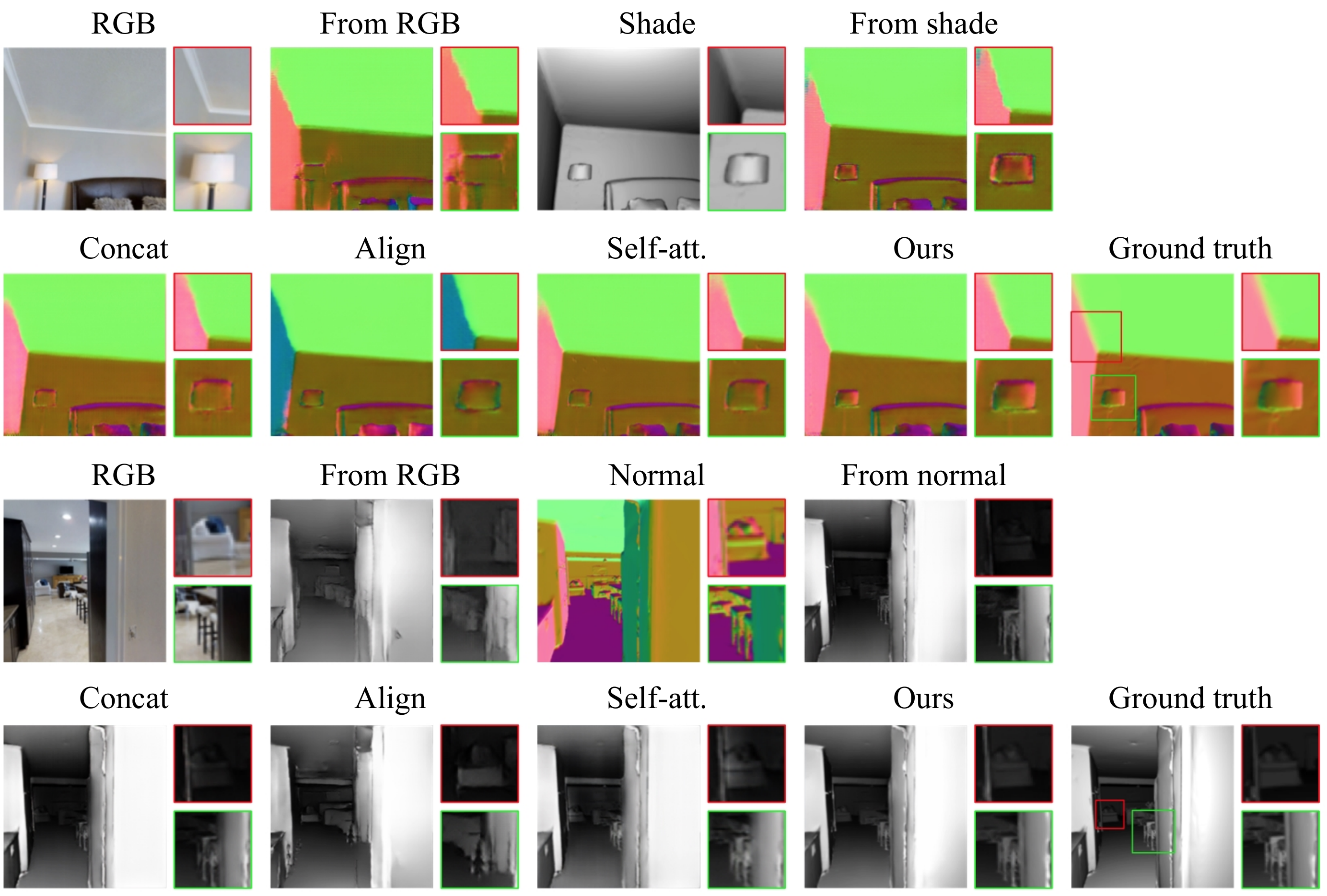}
\caption{Results comparison for image translation from RGB and Shade to Normal (upper group), and from RGB and Normal to Shade (lower group). Our fusion method again outperforms the other methods regarding both overall performance and details. }
\label{pic:trans_normal_shade}
\end{figure}

\begin{table}[h]
\centering

\caption{We compare training multimodal features in a parallel manner with different parameter sharing settings. Results of the proposed fusion method are reported at the last column. Evaluation metrics are FID/KID ($\times 10^{-2}$). We observe that the convolutional layers can be shared as long as we leave individual INs for different modalities, achieving even better performance.}
\resizebox{136mm}{!}{
\begin{tabular}{p{2.5cm} p{1.7cm}|p{2.6cm}<{\centering}p{2.6cm}<{\centering}p{2.6cm}<{\centering}|p{2.6cm}<{\centering}}

\toprule[1pt]
Modality&\makecell[l]{Network\\stream}&\makecell[c]{Unshared convs\\unshared INs}&\makecell[c]{Shared convs\\shared INs}&\makecell[c]{Shared convs\\unshared INs}&\makecell[c]{Multi-modal\\fusion}\\
\midrule

\multirow{3.5}{*}{\makecell[l]{Shade+Texture\\\;\;$\to$RGB}} 
& Shade&{102.21 / 5.25}&{112.40 / 5.58}&{100.69 / 4.51}&{72.07 / 2.32}\\
& Texture&{\;\;98.19 / 4.83}&{102.28 / 5.22}&{\;\;93.40 / 4.18}&{65.60 / 1.82}\\
\cmidrule(r){2-6}
&Ensemble&{\;\;92.72 / 4.15}&{\;\;96.31 / 4.36}&{\;\;87.91 / 3.73}&{62.63 / 1.65}\\
\midrule
\multirow{4.5}{*}{\makecell[l]{Shade+Texture\\+Depth\\\;\;$\to$RGB}} 
& Shade&{101.86 / 5.18}&{115.51 / 5.77}&{\;\;98.49 / 4.07}&{ 69.37 / 2.21}\\
& Texture&{\;\;98.60 / 4.89}&{104.39 / 4.54}&{\;\;95.87 / 4.27}&{64.70 / 1.73}\\
& Depth&{114.18 / 5.71}&{121.40 / 6.23}&{107.07 / 5.19}&{71.61 / 2.27}\\
\cmidrule(r){2-6}
&Ensemble&{\;\;91.30 / 3.92}&{100.41 / 4.73}&{\;\;84.39 / 3.45}&{58.35 / 1.42}\\
\midrule
\multirow{5.5}{*}{\makecell[l]{Shade+Texture\\+Depth+Normal\\\;\;$\to$RGB}} 
& Shade&{100.83 / 5.06}&{131.74 / 7.48}&{\;\;96.98 / 4.23}&{68.70 / 2.14}\\
& Texture&{\;\;97.34 /  4.77}&{109.45 / 4.86}&{\;\;94.64 / 4.22}&{63.26 / 1.69}\\
& Depth&{114.50 / 5.83}&{125.54 / 6.48}&{109.93 / 5.41}&{ 70.47 / 2.09}\\
& Normal&{108.65 / 5.45}&{113.15 / 5.72}&{\;\;99.38 / 4.45}&{67.73 / 1.98}\\
\cmidrule(r){2-6}
&Ensemble&{\;\;89.52 / 3.80}&{102.78 / 4.67}&{\;\;86.76 / 3.63}&{57.19 / 1.33}\\

\bottomrule[1pt]

\end{tabular}}
\label{tabs:supp_multimodal}
\end{table}

\begin{table}[h]
\centering

\caption{An Instance-Normalization layer consists of four components, including scaling factors $\bm{\gamma}$, offsets $\bm{\beta}$, running mean $\bm{\mu}$ and variance $\bm{\sigma}^2$. Following Table \ref{tabs:multimodal}, we further compare the evaluation results when using unshared $\bm{\gamma},\bm{\beta}$ only, and using unshared $\bm{\mu},\bm{\sigma}^2$ only. Evaluation metrics are FID/KID ($\times 10^{-2}$). We observe these four components of INs are all essential to be unshared. Besides, using unshared scaling factors and offsets seems to be more important.}
\resizebox{136mm}{!}{
\begin{tabular}{p{2.5cm} p{1.7cm}|p{2.6cm}<{\centering}p{2.6cm}<{\centering}|p{3cm}<{\centering}p{3cm}<{\centering}}

\toprule[1pt]
Modality&\makecell[l]{Network\\stream}&\makecell[c]{Unshared convs\\unshared INs}&\makecell[c]{Shared convs\\unshared INs}&\makecell[c]{Shared convs,$\bm{\gamma}$,$\bm{\beta}$\\unshared $\bm{\mu}$,$\bm{\sigma}^2$}&\makecell[c]{Shared convs,$\bm{\mu}$,$\bm{\sigma}^2$\\unshared $\bm{\gamma}$,$\bm{\beta}$}\\
\midrule

\multirow{4.5}{*}{\makecell[l]{Shade+Texture\\+Depth\\\;\;$\to$RGB}} 
& Shade&{101.86 / 5.18}&{\;\;98.49 / 4.07}&{ 107.86 / 5.53}&{ 105.29 / 5.29}\\
& Texture&{\;\;98.60 / 4.89}&{\;\;95.87 / 4.27}&{105.46 / 5.25}&{102.90 / 5.06}\\
& Depth&{114.18 / 5.71}&{102.07 / 4.89}&{118.35 / 6.07}&{114.35 / 5.80}\\
\cmidrule(r){2-6}
&Ensemble&{\;\;91.30 / 3.92}&{\;\;84.39 / 3.45}&{\;\;96.30 / 4.41}&{\;\;92.25 / 4.02}\\
\midrule
\multirow{5.5}{*}{\makecell[l]{Shade+Texture\\+Depth+Normal\\\;\;$\to$RGB}} 
& Shade&{100.83 / 5.06}&{\;\;96.98 / 4.23}&{113.56 / 5.65}&{102.74 / 5.17}\\
& Texture&{\;\;97.34 /  4.77}&{\;\;94.64 / 4.22}&{105.36 / 5.32}&{\;\;97.53 / 4.56}\\
& Depth&{114.50 / 5.83}&{109.93 / 5.41}&{ 119.31 /  6.20}&{ 112.73 /  5.60}\\
& Normal&{108.65 / 5.45}&{\;\;99.38 / 4.45}&{108.01 /  5.06}&{100.34 /  4.53}\\
\cmidrule(r){2-6}
&Ensemble&{\;\;89.52 / 3.80}&{\;\;86.76 / 3.63}&{\;\;95.56 / 4.64}&{\;\;89.26 / 3.91}\\

\bottomrule[1pt]
\end{tabular}}
\label{tabs:supp_unshare_in}

\end{table}

\begin{figure}[t]
\centering
\includegraphics[scale=0.23]{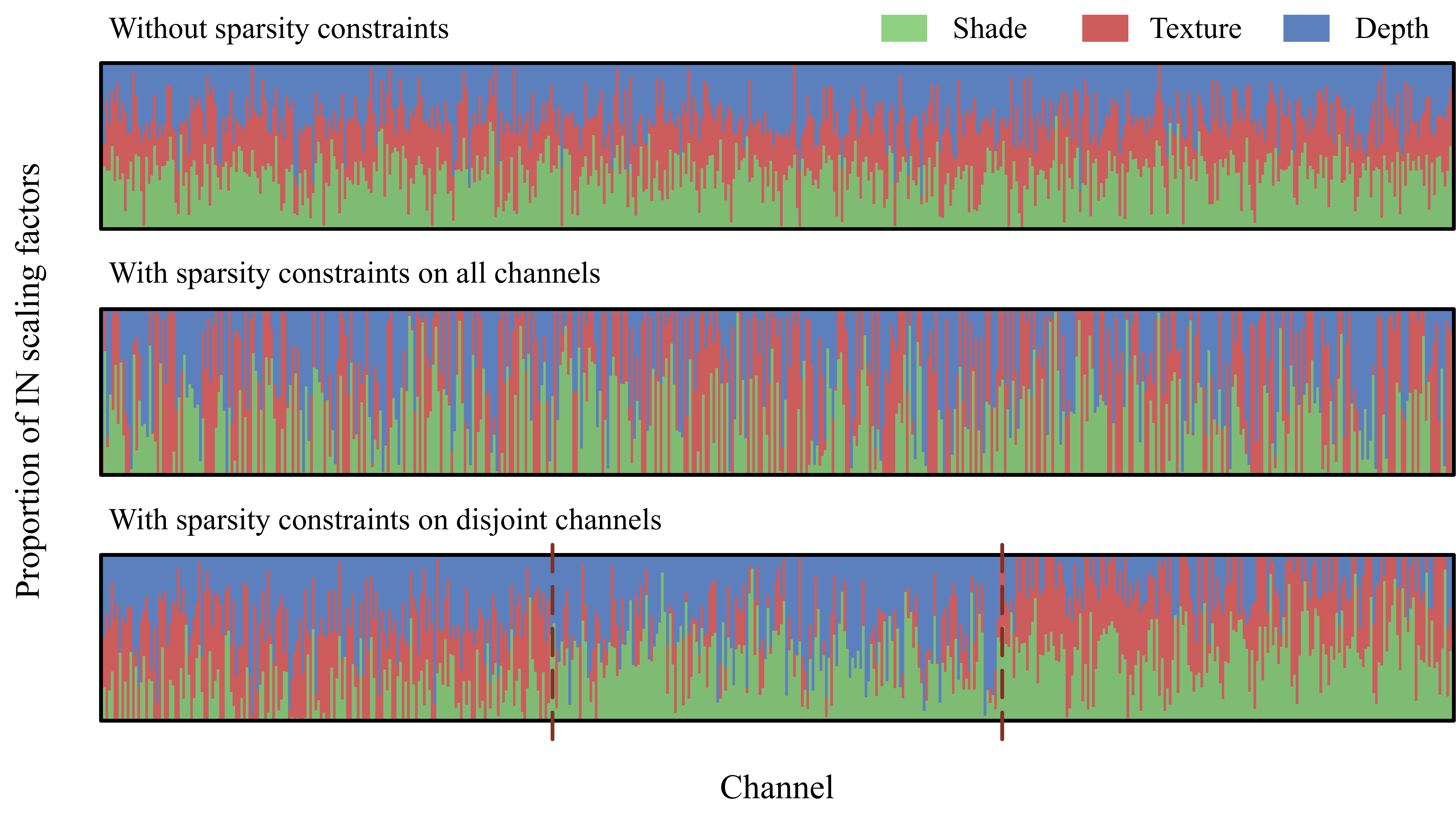}
\caption{We use shared convs and unshared INs, and plot the proportion of scaling factors for each modality, at the 7th conv layer, \emph{i.e.} $\gamma_c^{m,l,c}/(\gamma_c^{1,l,c}+\gamma_c^{2,l,c}+\gamma_c^{3,l,c}),$ where $m=1,2,3$ corresponding to Shade, Texture and Depth respectively, and $l=7$. \textbf{Top}: no sparsity constraints are applied, where the scaling factor of each modality occupies a certain proportion at each channel. \textbf{Middle}: sparsity constraints are applied to all channels, where scaling factors of one modality could occupy a large proportion, indicating the channels are re-allocated to different modalities under the sparsity constraints. Yet this setting is not very suitable for channel exchanging, as a redundant feature map of one modality may be replaced by another redundant feature map. \textbf{Bottom}: sparsity constraints are applied on disjoint channels, which is our default setting.}
\label{pic:featuremap_onelayer}
\end{figure}

\begin{figure}[t]
\centering
\includegraphics[width=\linewidth]{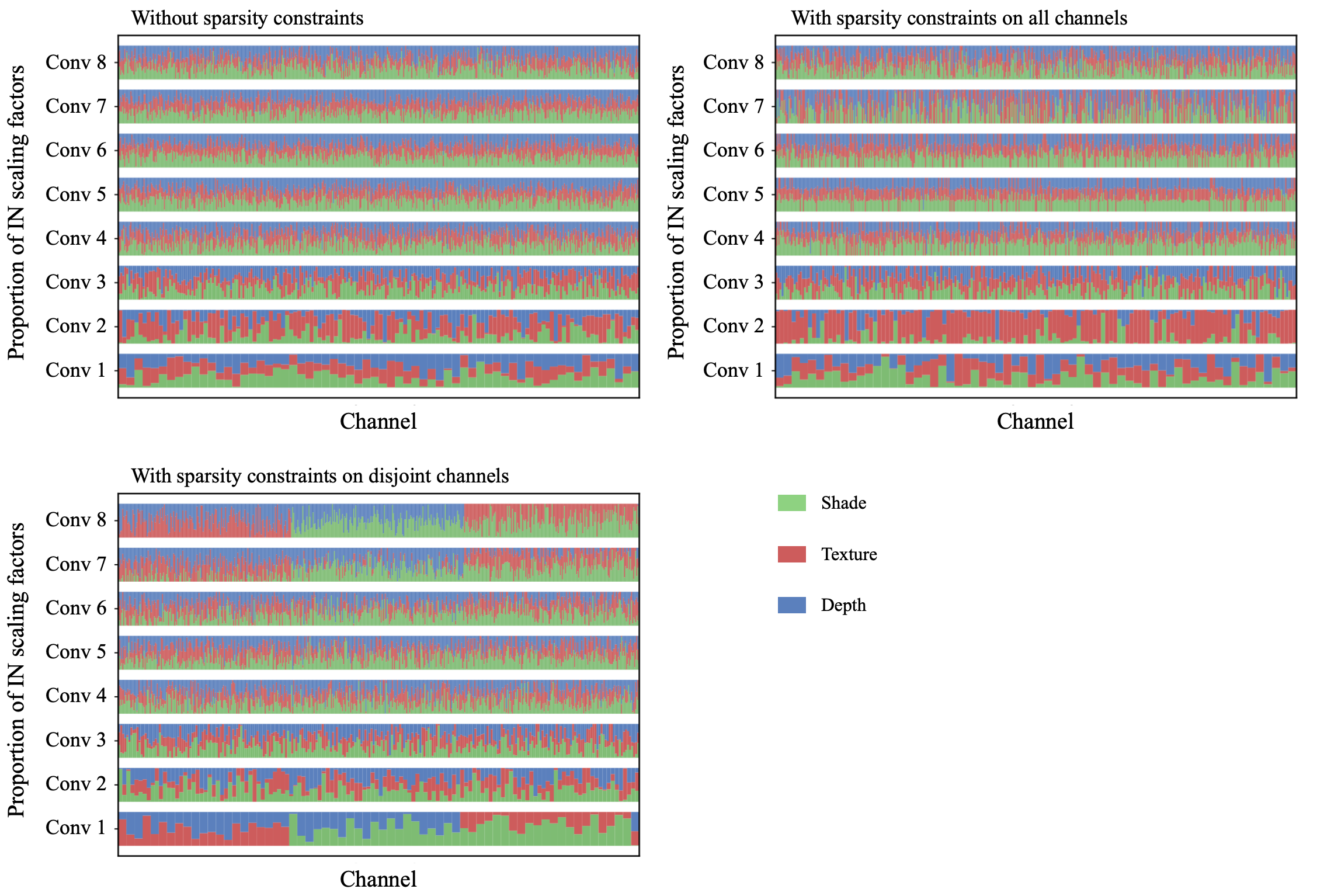}
\caption{Proportion of scaling factors in the U-Net encoder. We provide results at all layers. \textbf{Upper left}: no sparsity constraints are applied; \textbf{Upper right}: sparsity constraints are applied on all channels; \textbf{Bottom left}: sparsity constraints are applied on disjoint channels.}
\label{pic:featuremap_all}
\end{figure}

\begin{figure}[h]
\centering
\includegraphics[width=\linewidth]{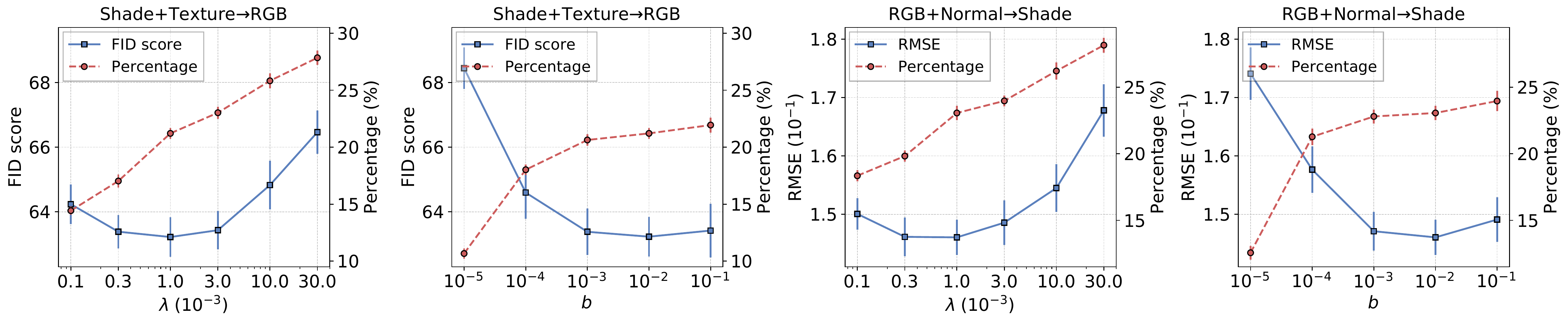}
\caption{Sensitivity analysis for $\lambda$ and $\theta$. In our channel exchanging process, $\lambda$ is the weight of sparsity constraint (Equation~\ref{eq:cen}), and $\theta$ is the threshold for choosing close-to-zero scaling factors (Equation~\ref{eq:exchange-bn}). We conduct five experiments for each parameter setting. In the 1st and 3rd sub-figures, $\lambda$ ranges from $0.1\times10^{-3}$ to $30.0\times10^{-3}$, and $\theta$ is set to $10^{-2}$. In the 2nd and 4th sub-figures, $\theta$ ranges from $10^{-5}$ to $10^{-1}$, and $\lambda$ is set to $10^{-3}$. The task name is shown at the top of each sub-figure. The left y-axis indicates the metric, and the right y-axis indicates the proportion of channels that are lower than the threshold $\theta$, \emph{}{i.e.} the proportion of channels that will be replaced. We observe that both hyper-parameters are not sensitive around their default settings ($\lambda=10^{-3}$ and $\theta=10^{-2}$).}
\label{sec:params_sensitivity}
\end{figure}

\end{document}